\documentclass[11pt]{article}
\usepackage[preprint]{acl}
\usepackage{times}
\usepackage{latexsym}
\usepackage{url}            
\usepackage{booktabs}       
\usepackage{amsfonts}       
\usepackage{nicefrac}       
\usepackage{microtype}      
\usepackage{xcolor}         
\usepackage{latexsym}
\usepackage{amssymb}
\usepackage{amsmath}
\usepackage{booktabs}
\usepackage{enumerate}
\usepackage{graphicx}
\usepackage{subfigure}
\usepackage{xspace}
\usepackage{float}
\usepackage{bbm}
\usepackage{bm}
\usepackage{multirow}
\usepackage{booktabs}
\usepackage{color}
\usepackage{framed}
\usepackage{stfloats}
\usepackage{amssymb}
\usepackage{pifont}
\usepackage{arydshln}
\usepackage{enumitem}
\usepackage{wrapfig}
\usepackage{algorithmic}
\usepackage{algorithm}
\usepackage{makecell}

\usepackage[T1]{fontenc}

\usepackage[utf8]{inputenc}

\usepackage[most]{tcolorbox}
\usepackage{tabularx}

\usepackage{hyperref}

\usepackage{microtype}

\usepackage{inconsolata}
\usepackage{multirow}
\usepackage[T1]{fontenc}
\usepackage{colortbl}
\definecolor{maroon}{cmyk}{0,0.87,0.68,0.32}

\usepackage{amsthm} %
\theoremstyle{plain}
\newtheorem{theorem}{Theorem}[section]
\newtheorem{proposition}[theorem]{Proposition}

\theoremstyle{definition}

\theoremstyle{remark}

\title{Mining or Synthesis? Rethinking Exploration Efficiency\\ in Iterative Alignment of Mathematical Reasoning}
\author{
Jun Rao$^{1}$\thanks{~Co-First Authors.},
 Zixiong Yu$^{2}$$^{*}$,~
 Xuebo Liu$^{1}$\thanks{~Corresponding Authors.},
 Guhan Chen$^{4}$,~
 Jing Li$^{1}$,~\\
 \bf{Hejin Wang}$^{3}$,~
 \bf{Jiansheng Wei}$^{2}$,~
  \bf{Xiaojun Meng}$^{2}$$^{\dagger}$, and
    \bf{Min Zhang}$^{1}$\\
    \textsuperscript{\rm1}Institute of Computing and Intelligence, Harbin Institute of Technology, Shenzhen, China \\
    \textsuperscript{\rm2}Huawei Large Model Data Technology Lab~~~\textsuperscript{\rm3}Huawei Multimodal Model Lab\\
    \textsuperscript{\rm4}Department of Statistics and Data Science, Tsinghua University, Beijing, China
    \\
    \texttt{rao7jun@gmail.com},~\texttt{\{yuzx19,wanghj20\}@tsinghua.org.cn},~\\\ \texttt{\{liuxuebo,zhangmin2021,li.jing\}@hit.edu.cn,~}\\
    \texttt{chen-gh23@mails.tsinghua.edu.cn},~\texttt{\{weijiansheng,xiaojun.meng\}@huawei.com}
    }

\begin{document}







\maketitle
\newcommand{\method}{PACE\xspace}
\begin{abstract}
Iterative Direct Preference Optimization (DPO) has emerged as a widely used paradigm for aligning Large Language Models on reasoning tasks. Existing approaches typically rely on Best-of-N sampling ($N\geq8$) to mine positive trajectories from the distribution tail. In this work, we show that in mathematical reasoning, increasing $N$ yields diminishing returns while increasing verifier-induced false-positive risk and the distribution shift required for policy updates.
To address this, we introduce PACE (Proximal Alignment via Corrective Exploration), a generation-based corrective framework that replaces exhaustive mining with low-budget exploration ($2\leq N\leq3$). Rather than searching for increasingly rare positive samples, PACE synthesizes high-fidelity preference pairs from failed explorations through corrective hindsight refinement and verification-guided filtering.
Empirically, PACE matches or exceeds the performance of DPO-R1 ($N=16$) while using about $1/5$ of the compute, and remains robust under 20\% label corruption, where high-$N$ baselines exhibit substantially higher noise exploitation.
\end{abstract}

\section{Introduction}
The alignment of Large Language Models (LLMs) has evolved from static Supervised Fine-Tuning (SFT) toward dynamic preference-based alignment methods. 
Unlike online reinforcement learning (RL) methods such as PPO~\cite{ppo} or GRPO~\cite{grpo}, iterative DPO (IDPO) optimizes an offline preference objective over iteratively collected preference pairs, making it a practical and widely adopted choice for large-scale alignment under constrained compute budgets~\cite{dpo,pang2024iterativereasoningpreferenceoptimization,yuan2024self,zhangonline,rao2026dynamicsamplingadaptsiterative}. In this setting, the model acts as both a learner and a data generator: it iteratively explores reasoning paths and updates its policy to increase the likelihood of preferred trajectories over rejected ones.

To construct high-quality pairs $(y_w, y_l)$ within this loop, a common practice is Best-of-N (BoN) sampling \cite{llama3,liu2024statistical}. This strategy relies on ``brute-force'' exploration: for a given prompt, the model generates a large pool of candidate responses (e.g., $N=16$). A verifier or reward model then acts as a filter, selecting a high-scoring trajectory as the winner $y_w$ and a lower-scoring trajectory as the loser $y_l$. The underlying assumption driving this approach is that larger exploration budgets yield higher-quality positive samples, which in turn improves alignment.

However, we challenge the optimality of this \emph{scaling exploration paradigm}. We argue that high-$N$ sampling not only increases computational cost but also amplifies systematic optimization risks. 
Prior work identifies that aggressive BoN exacerbates reward hacking~\cite{gao2023scaling} and induces distribution shift~\cite{lambert2023alignment}. Yet these analyses do not characterize the relationship between exploration budget and optimization risk. 
We show that larger $N$ increases the marginal false-positive risk of newly mined trajectories and the distribution shift required for policy updates. 
This reframes the question ``how large should $N$ be?'' as a risk-sensitive trade-off, motivating our design principle of reducing unnecessary exploration while preserving useful corrective signals.

Specifically, we identify two risks of high-$N$ sampling. First, BoN acts as an adversarial miner against the reward function. The marginal positive samples introduced by increasing $N$ are more likely to be false-positive trajectories, exploiting verifier noise or spurious correlations (Proposition~\ref{prop:fp}). Consequently, DPO-R1~\cite{zhangonline} may reinforce ``lucky guesses'' rather than robust reasoning. 
Second, even with a perfect verifier, mining $y_w$ from the extreme tail of the generation distribution, i.e., selecting trajectories that are found only after many sampling attempts, induces significant distribution shift. To learn such a trajectory, the model must update toward outlier regions with low probability support, thereby exacerbating optimization instability and excessive policy drift~\cite{NEURIPS2020_0d2b2061,lambert2023alignment}. We formalize this effect by deriving a KL divergence lower bound that scales as $\Omega(\log N)$ (Proposition~\ref{prop:kl}).

We propose PACE (Proximal Alignment via Corrective Exploration), a data-efficient framework that achieves strong alignment performance with a small sampling budget ($2\leq N \leq3$). Our core insight is that gradient quality matters more than sample perfection. 
Unlike DPO-R1, which selects heterogeneous samples from a large candidate pool, PACE synthesizes high-information pairs from failed attempts. When the model generates an incorrect solution during exploration, we rectify the reasoning path via hindsight correction. 
This process pairs the original error (Hard Negative) with the corrected path (Proximal Positive), forming a proximal preference pair. Because these two trajectories share semantic structure but diverge in logical validity, they provide a sharper, more stable gradient signal than BoN-mined DPO-R1 pairs. Our contributions are as follows:
\begin{itemize}[]
\item We formalize the trade-off between exploration breadth and robustness, showing that larger $N$ increases false-positive risks and distribution shifts (Propositions~\ref{prop:fp}--\ref{prop:kl}). 
\item We propose PACE, which replaces selection-based Best-of-N mining with a generation-based corrective strategy, achieving a 4--5$\times$ speedup over DPO-R1 ($N=16$) with comparable or better performance. 
\item Benchmarks show PACE matches or exceeds DPO-R1 ($N=16$). Notably, under 20\% label noise, PACE remains robust whereas DPO-R1 suffers from severe noise exploitation and poor stability.
\end{itemize}

\section{Related Work}

\paragraph{Iterative Alignment and the Limits of Best-of-N} Iterative DPO and related preference-based methods, such as SPIN~\citep{chen2024selfplayfinetuningconvertsweak} and Self-Rewarding LMs \citep{yuan2024self}, have become widely used paradigms for LLM alignment. A common strategy within this framework is BoN sampling (e.g., DPO-R1,~\citealp{zhangonline}), which filters high-quality trajectories from large candidate pools~\cite{llama,llama3,nakano2021webgpt}. While effective for static datasets, scaling $N$ in iterative settings introduces fundamental risks: \citet{gao2023scaling} demonstrated that aggressive BoN exacerbates reward hacking, and \citet{lambert2023alignment} identified severe distribution shift from tail sampling. 
We extend these qualitative insights by formally characterizing the scaling behavior of both false-positive amplification and distributional shift, motivating a shift from exploration breadth to logical density. 

\paragraph{Self-Correction and Data-Efficient Learning} 
Recent work explores leveraging model outputs for self-improvement. Methods such as STaR~\cite{zelikman2022star} and RestEM~\cite{singh2024beyond} adopt generate-and-filter approaches, which can be sample-inefficient because failures are discarded. Inference-time correction methods (Self-Refine,~\citealp{madaan2023selfrefine}; Reflexion,~\citealp{shinn2023reflexion}) incur significant deployment overhead without internalizing corrections into model weights. SAI-DPO~\cite{rao2026dynamicsamplingadaptsiterative} operates on data selection, whereas PACE operates on data synthesis. 
Unlike prior work, PACE treats failed explorations as Hard Negatives and synthesizes proximal positive trajectories via hindsight correction, yielding high-information preference pairs that improve alignment efficiency.







\section{Preliminaries}
\label{sec:prelim}

We formalize phenomena that undermine the effectiveness of high-$N$ Best-of-N sampling. All proofs are provided in Appendices~\ref{sec: theory}--\ref{Appendix: theory_KL}.

\subsection{Problem Formulation}
\label{sec:prelim:formulation}

Let $x$ be a problem, $y$ a trajectory sampled from policy $\pi_\theta$, and $R(y)\in\{0,1\}$ an ideal correctness indicator. A practical verifier $V$ has conditional false positive rate $\epsilon = \mathbb{P}(V=1 \mid R=0)$. Define:
\begin{itemize}[]
    \item \textbf{Intrinsic correctness:} $\alpha = \mathbb{P}(R=1)$;
    \item \textbf{Observed pass rate:} $\rho = \alpha + (1-\alpha)\epsilon$.
\end{itemize}
To simplify the analysis, we focus on marginal samples that first pass the practical verifier only when the sampling budget reaches $N$, corresponding to an empirical observed pass rate $\hat{\rho}=1/N$. Under smaller budgets, such samples would not provide a verified winner and thus could not form valid positive-negative preference pairs, so they would not enter the preference dataset; as $N$ increases, they are progressively introduced into training.

\subsection{False Positive Amplification}
\label{sec:prelim:fp}
For a solution passing the practical verifier $V$, its conditional false-positive probability is $\Psi(\alpha)=(1-\alpha)\epsilon/\rho$. Although this quantity does not explicitly depend on $N$ from a prior perspective, the dependence on $N$ emerges once we condition on the empirical pass rate of marginal samples.

\begin{proposition}[Simplified Marginal FP Monotonicity]
\label{prop:fp}
For marginal tasks with empirical pass rate $\hat\rho = 1/N$, the expected false-positive rate $\bar{\Psi}(N)$ is strictly increasing in $N$.
\end{proposition}
See Appendix~\ref{Appendix: theory_FP} for details. Intuitively, when a model repeatedly fails on a problem and only obtains a verifier-passing answer after many attempts, the posterior likelihood that this answer is a false positive becomes larger.




To validate the preceding theoretical prediction, we conduct an additional experiment. Table~\ref{tab:verifier_fpr} shows that the false-positive rate increases with the required sampling budget; meanwhile, LLM judges underestimate verifier noise. Specifically, we stratify prompts by the sampling budget required to obtain a mixed verifier outcome, randomly sample verifier-positive responses, and re-evaluate their correctness using human experts and LLM judges. See Appendix~\ref{sec:False-Positive Experiment} for the full experimental setup.






\begin{table}[h]
\centering
\scalebox{0.72}{
\begin{tabular}{@{}lccc@{}}
\toprule
\multirow{2}{*}{\textbf{Evaluator}} &
\multicolumn{3}{c}{\textbf{False-Positive Proportion (\%)}} \\
\cmidrule(l){2-4}
& \shortstack{\small{\textbf{2 attempts}}\\\small{\textbf{(Easy)}}} 
& \shortstack{\small{\textbf{3-4 attempts}}\\\small{\textbf{(Medium)}}} 
& \shortstack{\small{\textbf{5-8 attempts}}\\\small{\textbf{(Hard)}}}  \\
\midrule

\rowcolor[HTML]{F2F2F2}
\textbf{Human Evaluation} 
& \textbf{28.0} & \textbf{35.0} & \textbf{49.0} \\

DS-R1-Qwen-32B 
& 11.8 & 18.4 & 24.4 \\

Qwen2.5-72B-IT 
& 6.7 & 9.5 & 12.6 \\

\bottomrule
\end{tabular}
}
\caption{\textbf{False-Positive Rates.} The ``attempts'' column indicates that, within the corresponding sampling range, at least one verifier-positive and one verifier-negative response are observed, while under smaller budgets all sampled responses are verifier-negative. }
\label{tab:verifier_fpr}
\end{table}

\subsection{Trust Region Violation}
\label{sec:prelim:kl}
Even with a perfect verifier $V$ ($\epsilon=0$), a correct trajectory obtained only after many sampling attempts still lies in the tail of the reference generation distribution $\pi_{\mathrm{ref}}$. Therefore, making the updated distribution $\pi_{\mathrm{new}}$ achieve target correctness $\eta$ on a marginal sample whose correct trajectory is obtained only after $N$ sampling attempts requires the following distribution-shift cost:
\begin{proposition}[Simplified KL Lower Bound]
\label{prop:kl}
For tail trajectories with $\hat\rho = 1/N$, we have
\begin{equation*}
D_{\mathrm{KL}}(\pi_{\mathrm{new}} \| \pi_{\mathrm{ref}}) \;\gtrsim\; \eta \log N - H(\eta) + {\Omega}(1/N),
\end{equation*}
where $H(\eta)$ is the binary entropy function. 
\end{proposition}
See Appendix~\ref{Appendix: theory_KL} for details. Intuitively, forcing the model to learn increasingly difficult samples requires larger distribution-shift costs, which can become an important source of training instability; under smaller $N$, such samples would typically be excluded from the training set. Table~\ref{tab:kl} provides empirical support for this conclusion.









\section{Methods}
\begin{figure*}[t!]
\centering
\includegraphics[width=0.8\linewidth]{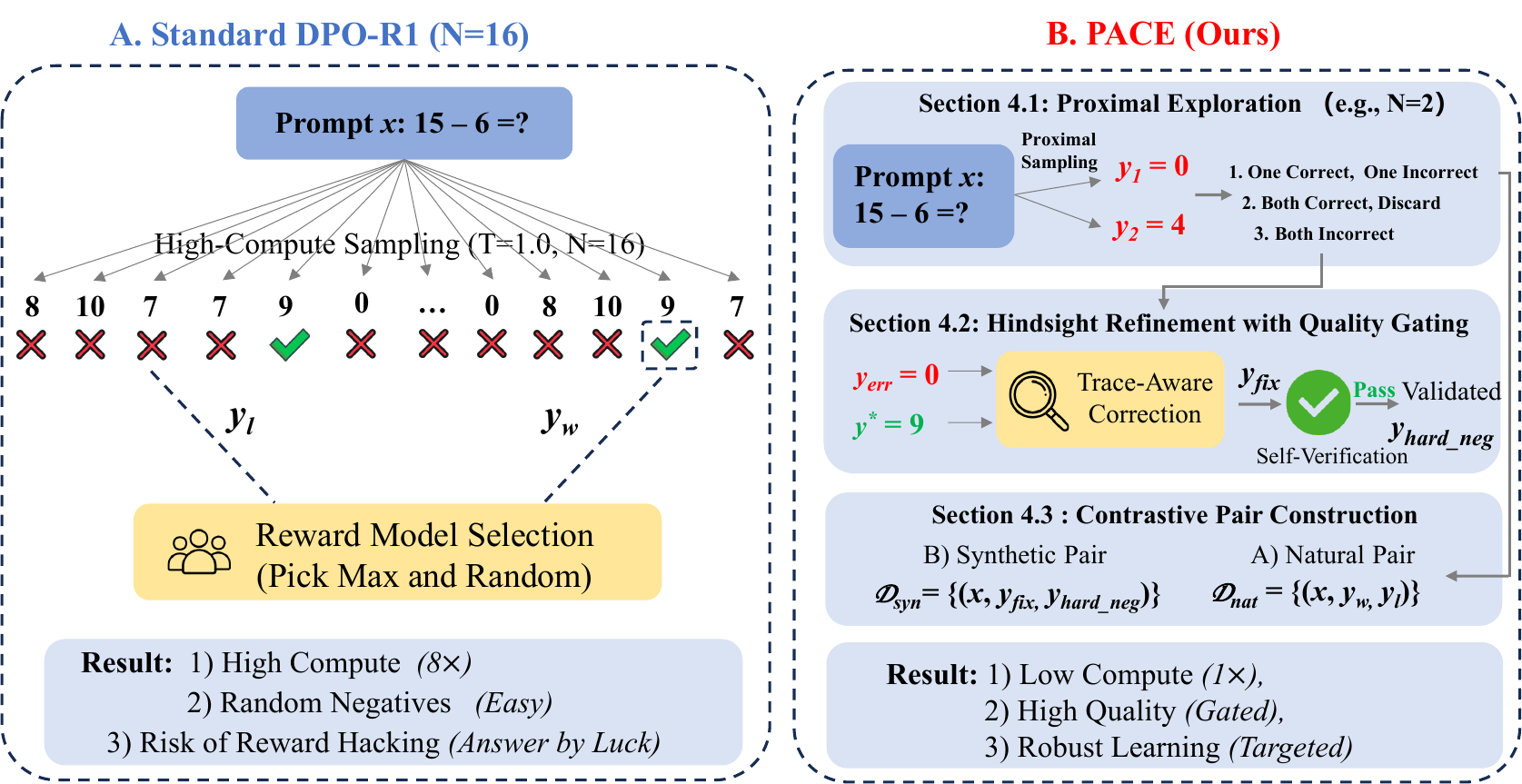}
\caption{Overview of PACE vs. Standard Best-of-N DPO. (A) Standard DPO-R1 ($N=16$): Relies on high-compute sampling to mine positive signals, risking reward hacking where models match labels through flawed logic.
(B) PACE (Ours): A three-phase pipeline: (I) Proximal Exploration ($N=2$) to minimize compute; (II) Hindsight Refinement with logical gating to synthesize verified corrections ($y_{fix}$) from failure traces ($y_{err}$); (III) Contrastive Construction of high-density pairs.
PACE achieves superior reasoning alignment with lower overhead and higher resistance to label noise by prioritizing logical density over search breadth.}
\label{fig:teaser}
\end{figure*}
In Figure~\ref{fig:teaser}, we present PACE, a framework designed to improve alignment efficiency by prioritizing logical density over search breadth. Guided by our theoretical analysis (\S\ref{sec:prelim}), PACE departs from the \textsc{Select-from-Many} paradigm of BoN, which relies on high-compute sampling to mine positive signals. Instead, it adopts a \textsc{Generate-and-Correct} paradigm. The full training loop is summarized in Algorithm~\ref{alg:pace} (Appendix~\ref{Algorithm}).

\subsection{Proximal Exploration}\label{sec:phase1}
To avoid the $\Omega(\log N)$ distribution shift inherent in high-$N$ sampling (Proposition~\ref{prop:kl}), PACE restricts the sampling budget to a small local probe of the policy's uncertainty. For a given prompt $x$, we generate exactly two trajectories $\mathcal{Y} = \{y_1, y_2\}$ from the current policy $\pi_\theta$. Generated via stochastic sampling with temperature $T=1.0$, this probes the local variance of the policy. 
This design is not merely a compute-saving heuristic; it is a \emph{structural consequence} of Proposition~\ref{prop:fp}. With $N=2$, the false positive amplification mechanism is fundamentally constrained: the probability of a marginal task entering the training set is bounded, and the expected verifier noise in incorporated samples remains minimal.

\subsection{Hindsight Refinement with Quality Gating}\label{sec:phase2}
\textbf{Trace-Aware Correction} 
Unlike standard self-correction which relies on generic prompts, we utilize a \emph{Trace-Aware} strategy. We provide the model with its own error trace $y_{{err}}$ and the ground truth label $y^*$. The instruction requires the model to identify the logical divergence between $y_{{err}}$ and $y^*$, and generate a corrected path $y_{{fix}}$.
\begin{equation}
    y_{fix} \sim \pi_\theta(\cdot \mid x, y_{err}, y^*, \text{instruction})
\end{equation}

\paragraph{The Quality Gate (Self-Verification)} The Consistency Filter is designed as a Causal Entailment check. By stripping the label and requiring deterministic re-generation, we verify that the reasoning trace $r_{\text{fix}}$ is the sufficient cause of the answer $y^*$, eliminating \emph{Rationalization} where the model hallucinates steps to justify a known outcome:
\begin{enumerate}[leftmargin=*,nosep]
    \item We strip the final answer from $y_{{fix}}$, retaining only the reasoning steps $r_{{fix}}$;
    \item We verify if the reasoning strictly entails the answer: $\mathbb{I}(G_{{greedy}}(x, r_{{fix}}) = y^*)$.
\end{enumerate}
Trajectories that fail this check are discarded, ensuring only grounded, logically valid corrections enter the training pool. This gate is critical: without it, the model would optimize toward hallucinated reasoning chains that happen to reach correct answers, reproducing the Lucky Guess phenomenon at the synthesis stage.



\subsection{Contrastive Pair Construction}\label{sec:phase3}
The core innovation of PACE lies in how it constructs the preference dataset $\mathcal{D}$ for DPO.
If both rollout samples are correct, we discard this type of sample; the remaining samples fall into two categories.
The final $y_{fix}$ has been obtained.
\\
\textbf{Scenario A: The Natural Pair (Exploration Success)} 
If one sample is correct ($y_w$) and the other incorrect ($y_l$) during the initial $N=2$ rollout, we construct a standard preference pair $(x, y_w, y_l)$. This captures data within the model's current competence, reinforcing known capabilities.
\begin{equation}
    \mathcal{D}_{nat} = \{ (x, y_w, y_l) \}
\end{equation}
\textbf{Scenario B: The Synthetic Pair (Exploration Failure)} If both initial samples are incorrect, we utilize the refined trajectory $y_{{fix}}$ from \S \ref{sec:phase2}. This indicates that the sample is located at a current defect in the model; performing supervised hindsight correction is more effective than random sampling for pulling the distribution back. We select the original incorrect sample, denoted as $y_{{hard\_neg}}$, to serve as the loser. The resulting pair is:
\begin{equation}
\mathcal{D}_{syn} = \{ (x, y_{fix}, y_{hard\_neg}) \}
\end{equation}
\paragraph{Why this matters (Gradient Efficiency)} In standard BoN, the winner and loser are often semantically distant (a random lucky guess vs. a random failure), leading to high-variance updates. In PACE, $y_{\text{fix}}$ is a \emph{Proximal Counterfactual} to $y_{\text{err}}$: they share the same prompt and potentially similar initial steps, diverging only at the critical error. Optimizing against such \textbf{Hard Negatives} provides a high-density gradient signal, effectively suppressing the specific error mode without disrupting the rest of the policy.

Finally, we use $\mathcal{L}_{\text{DPO}}$ optimizes the relative log-ratios of the pairs constructed in \S \ref{sec:phase3}:
\begin{equation*}
\begin{split}
    &\mathcal{L}_{\text{DPO}}(\pi_\theta, \pi_{\text{ref}}) = -\mathbb{E}_{(x, y_w, y_l) \sim \mathcal{D}}\\
 &\qquad\Big[\log \sigma\Big(\beta \log \tfrac{\pi_\theta(y_w|x)}{\pi_{\text{ref}}(y_w|x)} - \beta \log \tfrac{\pi_\theta(y_l|x)}{\pi_{\text{ref}}(y_l|x)}\Big)\Big].
\end{split}
\end{equation*}

\subsection{Complexity Analysis}\label{sec:efficiency_analysis}
Standard Best-of-N strategies suffer from static computational rigidity: the budget is fixed at $N$ regardless of prompt difficulty. In contrast, PACE employs a dynamic allocation strategy, expending additional compute only when epistemic uncertainty is high (i.e., exploration failure). Given that verification time is relatively negligible, we disregard this portion in our analysis.
Let $C_{\text{gen}}$ be the unit cost of generating one trajectory. The expected computational cost per prompt for PACE is:
\begin{equation}
\mathbb{E}[\text{Cost}_{\text{PACE}}] = 2 C_{\text{gen}} + P(\text{fail}) \cdot C_{\text{refine}},
\label{eq:cost}
\end{equation}
where $P(\text{fail})$ is the probability that the initial proximal exploration fails to yield a correct solution, and $C_{\text{refine}}$ is the refinement budget (typically $C_{\text{refine}} = C_{\text{gen}}$). This formulation yields three critical efficiency properties:
\begin{enumerate}[leftmargin=*,label=\textbf{\arabic*)},nosep]
    \item \textbf{Amortized Efficiency:} For easy instances where $P(\text{fail}) \to 0$, the cost converges to the lower bound $2 C_{\text{gen}}$. Standard BoN ($N=16$) remains statically inefficient at $16 C_{\text{gen}}$.
    \item \textbf{Bounded Worst-Case:} Even for hard failures ($P(\text{fail}) \to 1$), the cost is bounded by $3 C_{\text{gen}}$. With $C_{\text{refine}} = C_{\text{gen}}$, PACE achieves a $5.3\times$ speedup over the $N=16$ in the worst-case.
    \item \textbf{Self-Optimizing Dynamics:} Crucially, as the policy $\pi_\theta$ improves during DPO training, the failure rate $P(\text{fail})$ naturally diminishes. Consequently, PACE becomes progressively cheaper throughout the training lifecycle, whereas BoN costs remain constant.
\end{enumerate}

\begin{table*}[t]
\centering
\setlength{\tabcolsep}{4.5pt} 
\scalebox{0.73}{ 
\begin{tabular}{lccccccccc}
\toprule
 & \textbf{Compute} &\textbf{Trainable}&\textbf{
Processing}& \multicolumn{2}{c}{\textbf{Standard}} & \multicolumn{1}{c}{\textbf{Exam}} & \multicolumn{2}{c}{\textbf{Competition}} & \\
\cmidrule(lr){5-6} \cmidrule(lr){7-7} \cmidrule(lr){8-9}
\textbf{Method} & \textbf{Cost ($\downarrow$)}&\textbf{Pairs}&\textbf{Time($\downarrow$)}  & \textbf{Math500} & \textbf{Minerva} & \textbf{College} & \textbf{Olympiad}  & \textbf{AMC23} & \textbf{Avg.} \\
\midrule
\multicolumn{10}{c}{\textbf{\textit{Panel A: Llama-3.1-8B-Instruct}}} \\
\midrule
Instruct (Base) & - &-& -&29.0 & 16.5 & 21.3 & 7.3 & 15.3 & 17.9 \\
DPO-R1 (Low) & $N=2$&5,246&0.8 & 36.5\small{$\pm$1.4}&	\textbf{20.0\small{$\pm$1.3}}&24.3\small{$\pm$0.4}&	\textbf{9.7\small{$\pm$0.9}}		&18.0\small{$\pm$1.6}&21.7\small{$\pm$0.5} \\
DPO-R1 (Middle) &$N=4$&9,692&1.6&36.5\small{$\pm$1.7}&	17.4\small{$\pm$1.8}&23.9\small{$\pm$0.2}&	9.1\small{$\pm$0.3}	&	19.0\small{$\pm$1.4}&21.1\small{$\pm$0.8}\\
DPO-R1~\cite{zhangonline} & $N=8$&13,224 &2.4& 36.1\small{$\pm$1.1}&	15.9\small{$\pm$0.2}&23.4\small{$\pm$1.1}&	9.4\small{$\pm$1.0}	&	17.9\small{$\pm$0.2}&20.6\small{$\pm$0.5}\\
DPO-R1 (High) & $N=16$&16,241 &4.0& 33.6\small{$\pm$1.7}&	18.1\small{$\pm$0.9}&	23.9\small{$\pm$0.8}&	8.7\small{$\pm$0.3}	&	19.3\small{$\pm$0.5}&20.7\small{$\pm$0.3}
 \\
\rowcolor{gray!10} \textbf{PACE (Ours)} & $2<N<3$ &6,197&0.8& \textbf{36.6\small{$\pm$0.3}}	&19.6\small{$\pm$1.9}&	\textbf{24.6\small{$\pm$0.4}}&	9.4\small{$\pm$0.5}&		\textbf{20.6\small{$\pm$0.5}}&\textbf{22.2\small{$\pm$0.5}}

 \\
\midrule
\multicolumn{10}{c}{\textbf{\textit{Panel B: Qwen3-4B-Instruct}}} \\
\midrule
Instruct (Base) & - & -&-&84.6&	34.6&44.9&	47.4&		64.4&55.2 \\
DPO-R1 (Low) &$N=2$&3,479 &0.9& 84.2\small{$\pm$0.0}&	38.1\small{$\pm$0.9}&47.0\small{$\pm$0.2}&	50.2\small{$\pm$0.5}&63.8\small{$\pm$1.5}&56.7\small{$\pm$0.2}\\
DPO-R1 (Middle) &$N=4$&5,127&1.2&84.3\small{$\pm$0.1}&39.0\small{$\pm$0.0}&47.7\small{$\pm$0.1}	&	49.3\small{$\pm$1.0}		&67.1\small{$\pm$1.5}&57.5\small{$\pm$0.4}

 \\
DPO-R1~\cite{zhangonline} &$N=8$&8,403&2.8 & \textbf{85.1\small{$\pm$0.6}}&	38.3\small{$\pm$0.2}	&\textbf{48.0\small{$\pm$0.1}}&\textbf{50.4\small{$\pm$1.3}}	&	64.9\small{$\pm$3.4}&57.3\small{$\pm$0.4}\\
DPO-R1 (High) & $N=16$&8,964&5.2 &84.4\small{$\pm$0.0}&	\textbf{39.3\small{$\pm$1.3}}&		47.6\small{$\pm$0.1}&\textbf{50.4\small{$\pm$0.3}}&	67.2\small{$\pm$0.3}&\textbf{57.8\small{$\pm$0.3}}
\\

\rowcolor{gray!10} \textbf{PACE (Ours)} & $2<N<3$ & 10,717&1.0&83.1\small{$\pm$0.8}	&38.7\small{$\pm$0.9}	&47.3\small{$\pm$0.1}&49.9\small{$\pm$0.8}	&	\textbf{68.5\small{$\pm$0.7}}&57.5\small{$\pm$0.3}

 \\\midrule
 \multicolumn{10}{c}{\textbf{\textit{Panel C: Qwen3-8B-Instruct}}} \\
\midrule
Instruct (Base) & -&- &-& 81.2&	33.5&44.6&	49.9&	63.7&54.6
 \\
DPO-R1 (Low) &$N=2$&3,403&1.4 & \textbf{83.6\small{$\pm$0.3}}&	34.2\small{$\pm$0.6}&45.9\small{$\pm$0.1}	&49.3\small{$\pm$0.5}	&	\textbf{70.1\small{$\pm$2.3}}&56.7\small{$\pm$0.4}\\
DPO-R1 (Middle) &$N=4$&6,073&2.2&83.1\small{$\pm$0.6}	&36.4\small{$\pm$0.6}&	46.5\small{$\pm$0.1}&50.3\small{$\pm$0.2}	&	66.1\small{$\pm$0.8}&56.5\small{$\pm$0.2}

\\
DPO-R1~\cite{zhangonline} &$N=8$ &10,211&4.8& 83.3\small{$\pm$0.2}&	37.1\small{$\pm$1.0}	&\textbf{47.0\small{$\pm$0.1}}&50.3\small{$\pm$1.0}&		66.7\small{$\pm$1.2}&56.9\small{$\pm$0.0}

\\
DPO-R1 (High) & $N=16$&10,543&7.2&82.9\small{$\pm$0.2}&36.0\small{$\pm$0.0}	&46.7\small{$\pm$0.0}&\textbf{51.2\small{$\pm$0.6}}		&67.4\small{$\pm$1.1}&56.9\small{$\pm$0.2}\\
\rowcolor{gray!10} \textbf{PACE (Ours)} & $2<N<3$ &9,028& 1.7&83.2\small{$\pm$1.0}&	\textbf{38.7\small{$\pm$0.6}}&46.8\small{$\pm$0.0}&	49.4\small{$\pm$0.2}&		\textbf{70.1\small{$\pm$2.2}}&\textbf{57.6\small{$\pm$0.7}}

 \\
\bottomrule
\end{tabular}
}
\caption{\textbf{Main Results on Mathematical Reasoning.} We evaluate PACE against DPO-R1 baselines across varying sampling budgets ($N$).
We also report the average number of training samples and the processing time per sample (Processing Time), which includes the entire processing workflow time required to process a single sample.}
\label{tab:main_results_dual_model}
\end{table*}
\section{Experiments}\label{sec:Experiments}
\subsection{Setup}
We evaluate three state-of-the-art RL-tuned models: Llama-3.1-8B-Instruct and Qwen3-4B/8B-Instruct, representing the strongest reasoning capabilities under the 8B parameter class. Given that these base models have already been saturated with complex reasoning capabilities through intensive RL (GRPO/PPO), training on extended Think trajectories via DPO is computationally redundant for tasks that do not require long-form thought. 
Consequently, we adopt a ``No-Think'' mode to ensure efficient logical alignment while avoiding unnecessary overhead of over-optimizing intermediate reasoning tokens. We provide the generation, correction, and verification prompts in Appendix~\ref{app:prompts}, and qualitative case studies in Appendix~\ref{app:case_studies}.
We follow ~\citet{yang2024qwen25mathtechnicalreportmathematical} using common English
math benchmarks, details can be found in Appendix \ref{sec:datasets} and Appendix \ref{sec:train_detail}.
Methods such as STaR, Self-Refine, and RestEM date back to R1~\cite{deepseekai2025deepseekr1incentivizingreasoningcapability}. We compare current state-of-the-art methods:
\begin{itemize}[]
    \item \textbf{DPO-R1 (High, $N=8,16$):} The current industry standard for iterative alignment \cite{zhangonline}. For each prompt, we sample 8/16 trajectories at temperature $T=1.0$. The trajectory maximizing the verifier score is selected as $y_w$, and a random failing trajectory is selected as $y_l$. It represents the upper bound of high-resources.
    \item \textbf{DPO-R1 ($N=2,4$):} Compute-matched baselines restricted to the same sampling budget as PACE's exploration phase. They select the best of N samples without a corrective mechanism. These are ablations: comparisons against them demonstrate that PACE’s gains stem from the quality of corrective pairs, not merely from the quantity of exploration.
\end{itemize}

\subsection{Main Results}

We present evaluation results across five mathematical reasoning benchmarks in Table~\ref{tab:main_results_dual_model}, which reveal a fundamental limitation in standard iterative alignment and highlight the efficiency of PACE.

\paragraph{Efficiency} We compare the relative time required per sample (with 40,000 samples processed per iteration). In baselines, processing time increases with $N$. On Qwen3-4B, DPO-R1 (High, $N=16$) requires 5.2$\times$  per sample, while PACE completes in 1.0$\times$ : a $5.2\times$ acceleration that matches our theoretical worst-case bound of $5.3\times$ (Eq.~\ref{eq:cost}).

\paragraph{Effectiveness} PACE achieves significantly higher sample utilization efficiency than baseline models while maintaining minimal computational cost. On Llama-3.1-8B-Instruct, PACE achieves an average score of 22.2 by generating only 6,197 training pairs, significantly outperforming DPO-R1 (High), which consumed 16,241 pairs but only achieved 20.7 points. This reveals a failure mode of naive scaling: increasing the sampling budget $N$ grows the dataset size, yet yields lower final performance.
We attribute this to the \emph{Luck Trap}: brute-force sampling at high $N$ selects lucky guesses: solutions with correct answers but broken logic, from the distribution's tail.  
PACE employs Hindsight Refinement to transform discarded failed attempts into high-quality proximal refinement pairs.
On Qwen3-4B, PACE's score (57.5) is on par with DPO-R1 (High)'s 57.8, but PACE requires only 1.0$\times$ sampling time to generate 10,717 high-quality training pairs, whereas DPO-R1 (High) incurs 5.2$\times$ to generate 8,964 pairs. 
We attribute this to a saturation effect: Qwen3 models are already heavily RL-tuned, leaving limited headroom for iterative DPO refinement.  In such regimes PACE primarily serves as an efficient stabilizer, delivering comparable accuracy with 5$\times$ less compute and stronger noise resilience.

\subsection{Ablation Studies}
To dissect the contribution of each component, we conduct an ablation study on Llama-3.1-8B (Table~\ref{tab:ablation}). The results reveal that Trace-Aware Refinement is critical: replacing it with Label-Only guidance (providing only $y^*$ without $y_{{err}}$) degrades performance by $-4.2\%$. This confirms that conditioning on the specific error trace is essential to force comparative debugging and prevent the model from merely rationalizing the target label.

The catastrophic performance of SFT-only (11.0\%) compared to PACE (22.2\%) reveals a critical insight: for reasoning-saturated models, simply seeing correct paths is insufficient. The contrastive signal provided by the $(y_{{fix}}, y_{{err}})$ pair is mandatory to suppress specific logical error modes. Furthermore, substituting proximal pairs with random negatives results in a $-2.4\%$ drop, validating that optimization against hard, on-policy negatives yields higher-fidelity gradient signals than random baselines. Finally, removing the Consistency Filter causes a $-1.2\%$ decline, underscoring the necessity of quality gating to prevent hallucinated reasoning from polluting training dynamics.

\begin{table}[t]
\centering
\scalebox{0.65}{
\begin{tabular}{@{}llcc@{}}
\toprule
\textbf{Module} & \textbf{Ablation Variant} & \textbf{Acc. (\%)} & \textbf{$\Delta$} \\ \midrule
\textbf{PACE (Full)} & Trace-Aware + Proximal + Verified & \textbf{22.2} & - \\ \midrule
\textit{1. Refinement} & w/ Label-Only Guidance (No Trace) & 18.0 & $-4.2$ \\
& No Refinement (DPO-R1, $N=2$) & 21.7 & $-0.5$ \\ \midrule
\textit{2. Pairing} & w/ Random Negatives & 19.8 & $-2.4$ \\ 
&w/ SFT Only (on $y_{fix}$) &11.0&$-$11.2\\
\midrule
\textit{3. Gating} & w/o Consistency Filter &21.0 & $-1.2$ \\ \bottomrule
\end{tabular}
}\caption{\textbf{Ablation Study.} 
We isolate the contribution of each PACE component with a standard baseline variant. 
}
\label{tab:ablation}
\end{table}
\section{Mechanistic Analysis}
We analyze PACE's optimization behavior to validate the theoretical mechanisms in \S\ref{sec:prelim}. All experiments use Llama-3.1-8B unless specified. 
\subsection{Optimization Dynamics}
\label{sec:analysis:dynamics}

\begin{figure}[t]
    \centering
\includegraphics[width=0.9\linewidth]{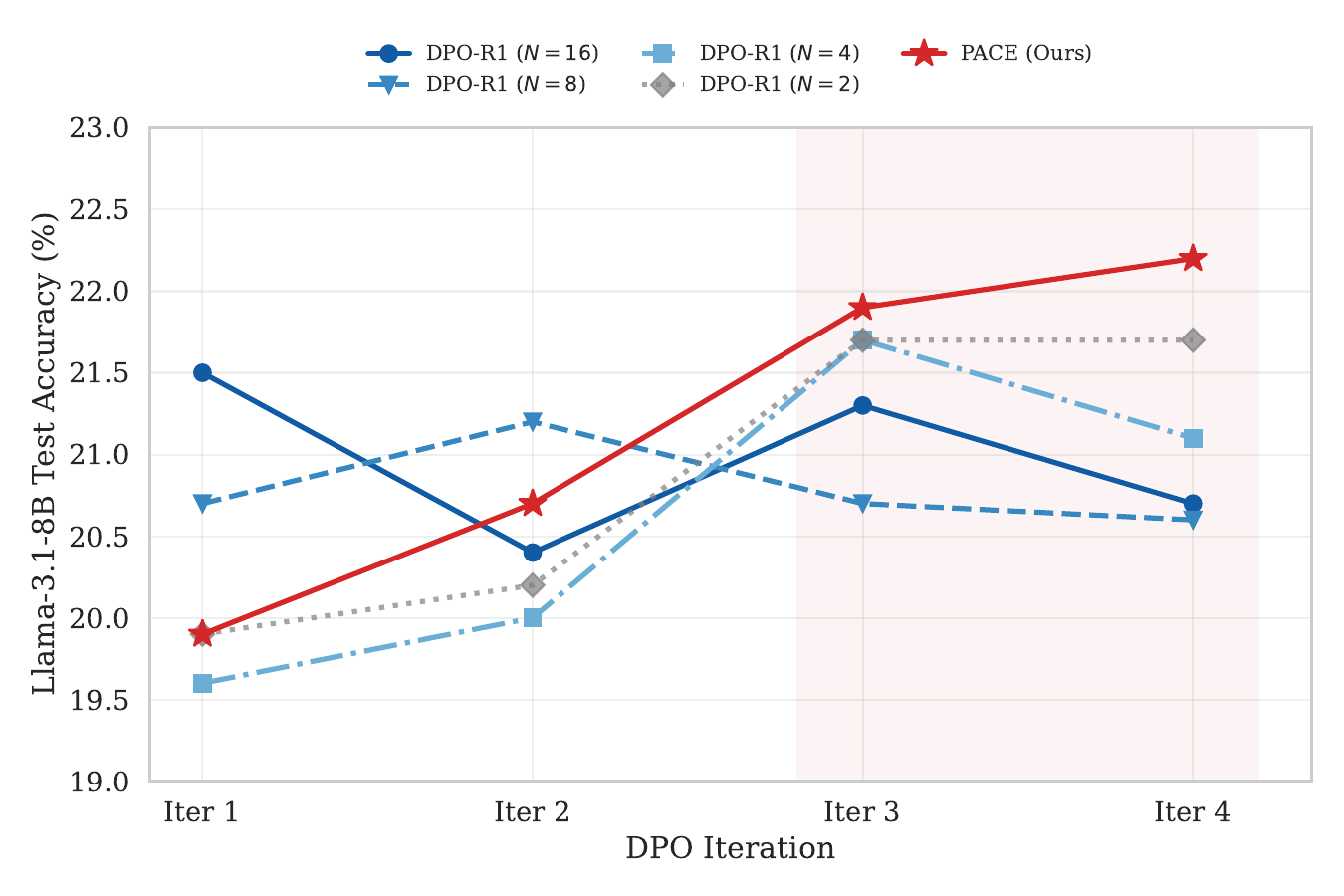}
    \caption{The Dynamics of Iterative Alignment.}
\label{fig:iterative_scaling_dynamics}
\end{figure}
Figure~\ref{fig:iterative_scaling_dynamics} reveals a fundamental quality-quantity trade-off in iterative alignment. DPO-R1 (High, $N=16$) exhibits classic overfitting: strong initial performance (21.5\% at Iter 1) degrades to 20.7\% by Iter 4. We attribute this to the \emph{unsupervised nature of rejection sampling}: as the model improves, $N=16$ retrieves increasing spurious positives (hallucinated reasoning with correct answers) from the distribution tail, poisoning the reward signal. 
In contrast, PACE demonstrates robust scaling: starting lower (19.8\% at Iter 1) due to the Consistency Filter rejecting low-confidence lucky guesses, it achieves \emph{monotonic improvement} to 22.2\% at Iter 4. This confirms that for iterative learning, \emph{signal purity} (ensured by refinement and filtering) outweighs \emph{exploration volume}.

\paragraph{Distribution Shift} To bridge the gap between theoretical lower bounds and global policy behavior, we measured the empirical KL divergence ($\mathbb{D}_{KL}(\pi_{\theta} \| \pi_{ref})$) on the MATH500 dataset for Llama-3.1-8B. Table~\ref{tab:kl} validates Proposition~\ref{prop:kl}: As the sampling budget $N$ increases, the policy drifts significantly far from the reference with diminishing returns in accuracy.  
DPO-R1 ($N=16$) exhibits $2.3\times$ the empirical KL divergence of PACE (0.0071 vs 0.0029), confirming that high-$N$ sampling forces updates toward outlier trajectories.


\begin{table}[t]
\centering
\small
\begin{tabular}{lcc}
\toprule
\textbf{Method} & $N$ & $D_{\mathrm{KL}}(\pi_\theta \| \pi_{\text{ref}})$ \\
\midrule
PACE & 2$\sim$3 & 0.0029 \\
DPO-R1 & 2 & 0.0031 \\
DPO-R1 & 4 & 0.0045 \\
DPO-R1 & 8 & 0.0058 \\
DPO-R1 & 16 & 0.0071 \\
\bottomrule
\end{tabular}
\caption{Empirical KL divergence from reference policy.}\label{tab:kl}
\end{table}

\paragraph{Semantic Topology} 
Instead of analyzing gradients directly (which is noisy), we analyze the geometric relationship between the winner and loser responses as shown in Figure~\ref{fig:similarity}. We randomly sampled 1,000 pairs from each dataset (PACE refinement and BoN) and computed the Cosine Similarity between the embeddings of the winner and loser trajectories
using a pre-trained sentence encoder. 
PACE pairs concentrate at high cosine similarity ($\mu \approx 0.95$), whereas DPO-R1 ($N=16$) spreads across lower values. 
High similarity forces the model to discriminate via \emph{fine-grained logical differences} rather than surface heuristics.

\begin{figure}
    \centering
\includegraphics[width=0.8\linewidth]{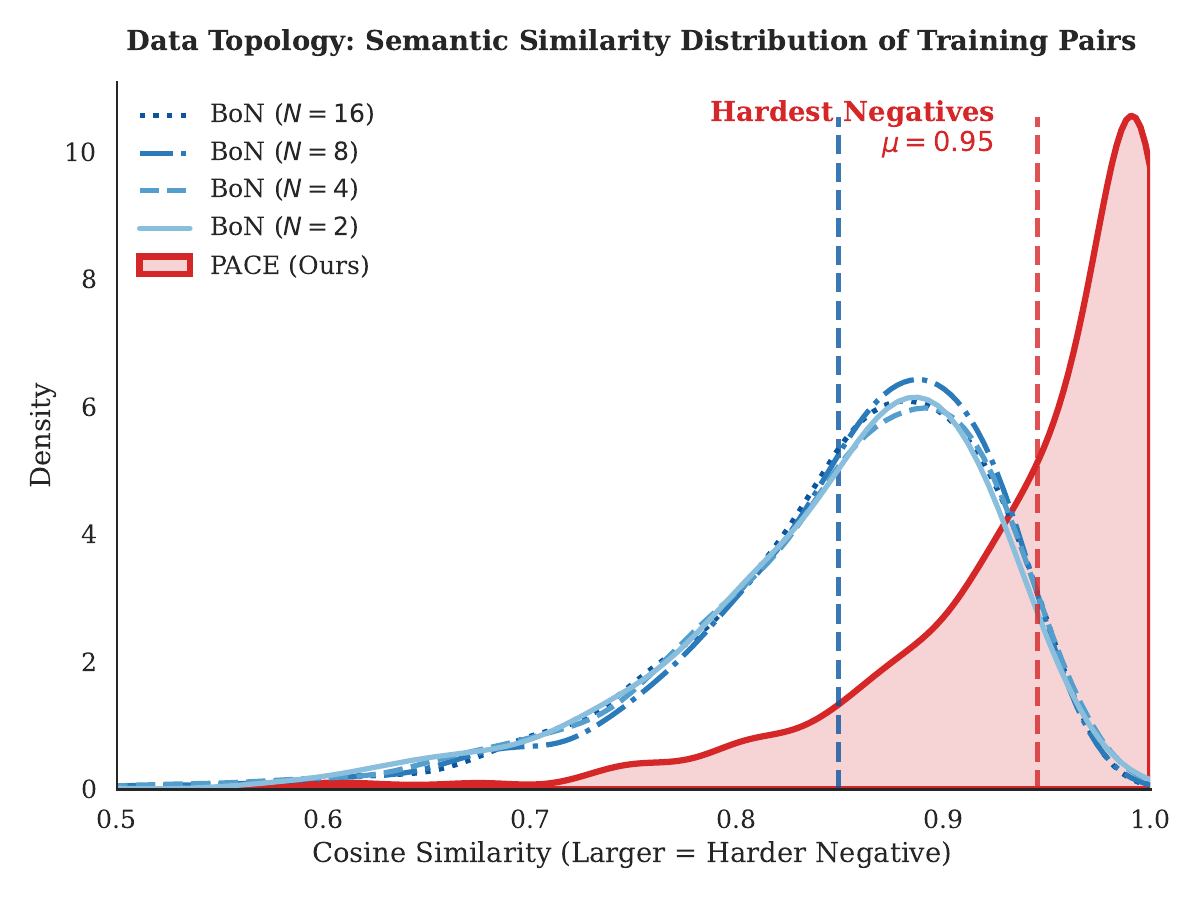}
    \caption{Semantic similarity distribution of training pairs. PACE synthesizes Hard Negatives (high similarity); BoN produces Easy Negatives (low similarity).} 
\label{fig:similarity}
\end{figure}
\subsection{Robustness to Label Noise}
\begin{table*}[t]
\centering
\small
\begin{tabular}{@{}lcccccccc@{}}
\toprule
 & \multicolumn{4}{c}{\textbf{Accuracy (Iter 1-4) $\uparrow$}} & \multicolumn{4}{c}{\textbf{Noise Diagnostics (Snapshots)}} \\ \cmidrule(lr){2-5} \cmidrule(lr){6-9} 
\textbf{Method} & \textbf{It. 1} & \textbf{It. 2} & \textbf{It. 3} & \textbf{It. 4} & \textbf{Yield (Clean)} & \textbf{Yield (Noisy)} & \textbf{Drop \%} & \textbf{NHR $\downarrow$} \\ \midrule
DPO-R1 ($N=2$) (Baseline) &\textbf{20.0}& 19.1 & 16.8 & 12.7 & 4,221 & 3,722 & 11.8\% & \textbf{4.3\%} \\
DPO-R1 ($N=16$) (Scaling) & 19.5 & \textbf{20.3}& \textbf{20.8} & 19.1 & 15,719 & 14,823 & 5.7\% & 28.1\% \\
\rowcolor[HTML]{EFEFEF} 
\textbf{PACE (Ours)} & 18.3 & 19.1 & 20.1 & \textbf{20.4} & 6,159 & 5,317 & 13.7\% & 4.9\% \\ \bottomrule
\end{tabular}
\caption{Iterative Performance and Noise Robustness Analysis on Llama-3.1-8B (20\% Label Noise). The left panel shows the accuracy evolution across 4 iterations. The right panel provides a diagnostic snapshot of the training buffer: \textbf{Yield} refers to the number of training pairs, and \textbf{Noise Hit Rate (NHR)} measures the percentage of samples where the model successfully mined a trajectory matching the corrupted ``0'' label. 
}
\label{tab:noise_evolution}
\end{table*}

To validate Proposition~\ref{prop:fp}, we corrupt 20\% of ground-truth labels by replacing the final answer with ``0'', creating a Small-Digit Trap that tests whether the method can exploit noisy but easily matchable labels (Table~\ref{tab:noise_evolution}).
DPO-R1 (High) achieves higher initial accuracy but suffers from 28.1\% Noise Hit Rate (NHR): the percentage of samples where the model successfully mines trajectories matching corrupted labels. This confirms that high-$N$ exploration acts as an \emph{adversarial miner}, synthesizing hallucinations to match erroneous signals. Despite initial buffering by correct samples, long-term policy degradation is inevitable.
DPO-R1 ($N=2$) exhibits low NHR (4.3\%) but severe signal starvation (Yield Drop 11.8\%), stalling at 12.7\% accuracy: the cost of insufficient exploration.
PACE resolves this trade-off: low NHR (4.9\%) comparable to $N=2$, yet 42.8\% more valid signals recovered via hindsight refinement. The 13.7\% yield drop is a \emph{deliberate logical purge}: the Consistency Filter rejects paths irreconcilable with corrupted labels. This results in the highest iterative performance, proving that \emph{synthesized purity outperforms mined volume} under noise.

\subsection{Scaling and Generalization}\label{sec:analysis:efficiency}
\begin{table}[t]
\centering
\scalebox{0.7}{
\begin{tabular}{@{}lcc@{}}
\toprule
\textbf{Method} & \textbf{Exploration ($N$)} & \textbf{Accuracy (\%)} \\ \midrule
DPO-R1 & 2 & 21.3 \\
DPO-R1 & 4 & 21.3 \\
DPO-R1 & 8 & 21.4 \\
DPO-R1 & 16 & 21.4 (\textcolor{blue}{$\uparrow$ 0.1}) \\ \midrule
\rowcolor[HTML]{F2F2F2} 
\textbf{PACE (Ours)} & \textbf{2} & \textbf{21.9} (\textcolor{red}{$\uparrow$ 0.6}) \\ \bottomrule
\end{tabular}
}
\caption{Iso-Data Efficiency: all methods trained on 5,000 pairs (Llama-3.1-8B).}
\label{tab:iso_data}
\end{table}
\paragraph{Iso-Data Analysis.} Fixing the training budget to 5,000 pairs regardless of $N$ (Table~\ref{tab:iso_data}), DPO-R1 exhibits non-monotonic behavior: $N=4$ outperforms $N=2$, but $N=16$ regresses. This validates adversarial mining, without logic-preserving constraints, larger pools capture deceptive trajectories that misalign reasoning with answers.
PACE achieves peak accuracy, outperforming all DPO-R1 variants within the same budget. This provides definitive evidence that \emph{construction quality (Synthesis) dominates exploration breadth (Mining)}.

\begin{table}[t]
\centering
\small
\setlength{\tabcolsep}{3.5pt}
\begin{tabular}{lcccccc}
\toprule
& \multicolumn{4}{c}{\textbf{Rollout}} & \textbf{PACE} \\
\cmidrule(lr){2-5} \cmidrule(lr){6-6}
\textbf{Configuration} & $N$=2 & $N$=4 & $N$=8 & $N$=16 & (Ours) \\
\midrule
GRPO (Think) & 33.4 & 32.6 & \textbf{34.8} & 34.3 & 34.4 \\
DPO (No-Think)& 33.7 & 34.5 & 35.5 & \textbf{36.3} & 34.1 \\
DPO (Think) & 37.6 & 40.4 & \textbf{40.6} & 39.9 & \textbf{40.6} \\
\bottomrule
\end{tabular}
\caption{PACE vs Online RL (GRPO/Think) and DPO (Think/No-Think) across multiple training configurations on Qwen3-0.6B. We show average accuracy of 5 benchmarks (\%) under varying sampling budgets. }\label{tab:qwen06b_results}
\end{table}

\paragraph{Model Scalability} 
Table~\ref{tab:qwen06b_results} validates PACE's cross-paradigm stability on Qwen3-0.6B across GRPO (Think), DPO (No-Think), and DPO (Think) configurations. The detailed experimental setup is described in Appendix \ref{app:setup}.
PACE consistently approaches or matches the peak accuracy of high-$N$ baselines, achieving 34.4\% vs. GRPO's best 34.8\% ($N=8$), and 40.6\% matching DPO (Think)'s best 40.6\% ($N=8$), while operating at a fixed exploration budget of  $N\approx2$. Notably, baselines suffers from performance regression at $N=16$ in both GRPO (34.3\%) and DPO-Think (39.9\%), whereas PACE avoids this degradation by synthesizing high-fidelity proximal pairs rather than mining distribution tails. These results confirm that corrective synthesis generalizes beyond the DPO setting, delivering comparable alignment efficiency under diverse training regimes.

\paragraph{Cross-Domain Preservation} Table~\ref{tab:ood_generalization} tests whether mathematical alignment sacrifices general capabilities. DPO-R1 ($N=8$) degrades on BBH~\cite{bbh}, MMLU~\cite{MMLU}, and TruthfulQA~\cite{lin-etal-2022-truthfulqa} versus $N=2$: classic alignment tax from adversarial mining. PACE preserves distributional integrity: its average OOD score outperforms $N=8$ and demonstrates positive transfer on HumanEval (+1.4\%), confirming that PACE enhances reasoning without compromising foundational intelligence.

\begin{table}[t]
\centering
\scalebox{0.65}{
\begin{tabular}{@{}lccccc@{}}
\toprule
\textbf{Method} & \textbf{BBH} & \textbf{Code} & \textbf{MMLU} & \textbf{TruthfulQA} & \textbf{Avg.} \\ \midrule
Llama3.1-8B & 56.2 & 51.7 & 68.6 & 37.7 & 53.5 \\ \midrule
DPO-R1 ($N = 2$) & 58.1 & 53.9 & 68.4 & 38.8 & 54.8 \\
DPO-R1 ($N = 8$) & 57.0 ({$\downarrow$}) & 52.7 ({$\downarrow$}) & 68.4 ({--}) & 37.9 ({$\downarrow$}) & 54.0 \\ \midrule
\rowcolor[HTML]{F2F2F2} 
\textbf{Ours} & \textbf{59.1} (\textcolor{red}{$\uparrow$}) & \textbf{55.3} (\textcolor{red}{$\uparrow$}) & \textbf{68.3} ($\downarrow$) & \textbf{38.2} (\textcolor{red}{$\uparrow$}) & \textbf{55.2} \\ \bottomrule
\end{tabular}
}\caption{\textbf{Out-of-Domain Generalization.}  While high-$N$ exploration often leads to alignment tax on OOD tasks, PACE maintains distributional integrity.}
\label{tab:ood_generalization}
\end{table}

\section{Conclusion}
We revisited the widespread assumption that scaling inference compute is necessary for alignment quality. 
Our results reveal that aggressive Best-of-N sampling incurs hidden costs: it degrades iterative stability, inflates distribution shift, and amplifies vulnerability to noisy rewards. 
We introduced PACE, which replaces brute-force mining with hindsight correction to synthesize high-quality preference pairs from minimal exploration. Across mathematical reasoning benchmarks, PACE matches or exceeds the performance of high-budget baselines while reducing computational overhead by over $5\times$, and maintains robustness even when 20\% of labels are corrupted. 
These findings suggest a practical shift from exhaustive answer mining to targeted signal synthesis, offering a more accessible path for training reasoning models at scale. 

\section*{Limitations}
There are several limitations to our work. First, due to training resource constraints, our algorithm primarily focuses on offline RL methods such as iterative DPO. We only conduct experiments on mainstream online approaches like GRPO
for small-scale models (Qwen3 0.6B). Online RL methods, such as PPO or GRPO, require continuous environment rollouts for each training step. This persistent interaction with the environment substantially increases computational costs, which is orthogonal to PACE’s goal of minimizing data-generation overhead.
Given that iterative DPO is a low-resource setting for large-scale post-training, we prioritize validating PACE in this practically relevant regime. 
Another limitation is that we have only conducted experiments on test sets of varying difficulty levels within the field of mathematical reasoning, and other domains such as code and general domains remain unexplored. Future work will extend this corrective paradigm to open-ended domains beyond mathematics.

\section*{Impact Statement}
As LLM post-training scales, the computational cost of data generation becomes a bottleneck. PACE offers a sustainable path forward, demonstrating that smarter data synthesis can replace larger data mining. This has significant implications for democratizing the training of reasoning models, particularly in domains where verified data is scarce or verifiers are imperfect. While our current work focuses on mathematical reasoning with ground-truth labels, future directions include extending the corrective mechanism to open-ended generation tasks using AI-based feedback (e.g., LLM-as-a-Judge).




\bibliography{example_paper}

\newpage
\appendix
\section{The Risks of Best-of-N}\label{sec: theory}
We analyze the learning dynamics of Iterative DPO under the BoN sampling strategy. While scaling inference compute (increasing $N$) is beneficial for test-time performance, we argue that naively using it for training data generation introduces two fundamental pathologies: False Positive Amplification (susceptibility to reward noise) and Distributional Shift (violating the trust region).

Let $\pi_{\text{ref}}(y|x)$ be the reference policy (i.e., the policy from the previous iteration), and $\mathcal{V}(y, x)$ be a verifier (or reward function) approximating the ground truth $R(y, x)$. In standard DPO-R1, for a prompt $x$, we sample $\mathcal{Y}_N = \{y_i\}_{i=1}^N \sim \pi_{\text{ref}}(\cdot|x)$ to construct preference pairs $(y_w, y_l)$ via:
\begin{equation}
    y_w \in \mathcal{Y}_W =  \arg{\max_{y \in \mathcal{Y}_N}} \mathcal{V}(y, x), \quad y_l \in \mathcal{Y}_N \setminus \mathcal{Y}_W.
\end{equation}
The policy is then updated to minimize the DPO loss $\mathcal{L}_{\text{DPO}}$.

\subsection{False Positive Amplification}\label{False Positive}
In real-world scenarios, the verifier $\mathcal{V}$ is rarely a perfect oracle. It may contain noise (e.g., labeling errors) or structural biases (e.g., reward hacking spots). We demonstrate that increasing $N$ turns the sampling process into an adversarial miner against the verifier.
Consider a challenging problem $x$. We partition the solution space into the true solution set $\mathcal{T} = \{y \mid R(y)=1\}$ and the hallucinated set $\mathcal{H} = \{y \mid R(y)=0, \mathcal{V}(y)=1\}$. The generation process is modeled as follows: the model produces a valid reasoning path with probability $\alpha = P(y \in \mathcal{T} \mid x)$. If the reasoning fails, the model may still trigger a verifier flaw with a conditional intrinsic defect rate $\epsilon = P(y \in \mathcal{H} \mid y \notin \mathcal{T})$. Consequently, the verifier pass probability $\rho = P(\mathcal{V}(y)=1)$ is $\rho = \alpha + (1 - \alpha)\epsilon$.
Since the verifier cannot distinguish between true and spurious solutions, the probability that a selected positive sample is a false positive (FP) is given by:
\begin{align}\label{eq: FP_prior} 
\Psi(\alpha)=P(y \in \mathcal{H} \mid \mathcal{V}(y)=1) = \tfrac{(1 - \alpha)\epsilon}{\alpha + (1 - \alpha)\epsilon}.
\end{align}
From a prior perspective, the FP probability \eqref{eq: FP_prior} is formally independent of $N$. However, since DPO training necessitates the construction of positive-negative pairs, increasing $N$ effectively induces a non-random drift in the distribution of training samples: as multiple attempts raise the probability of passing the verifier, it introduces challenging tasks that would rarely yield a correct solution under smaller $N$. These newly incorporated samples typically exhibit an extremely low empirical pass rate $\hat{\rho}$. For instance, when the sample budget increases from $N-1$ to $N$, the marginal observed pass rate for a newly activated sample is merely $1/N$.
Based on this observation, we employ Maximum Likelihood Estimation (MLE) to perform a posterior correction on \eqref{eq: FP_prior}. 
Specifically, the MLE of the model’s intrinsic reasoning capability $\hat{\alpha}$ is given by (see Appendix \ref{Appendix: theory_FP} for details):
\begin{align}\label{eq: MLE_alpha}
\hat{\alpha} 
= \Pi_{[0, 1]} \left( \tfrac{\hat{\rho} - \epsilon}{1 - \epsilon} \right) 
= \max\left(0, \tfrac{\hat{\rho} - \epsilon}{1 - \epsilon}\right),
\end{align}
where $\Pi_{[0,1]}$ denotes the projection onto the interval $[0,1]$. 
Combining the monotonicity of \eqref{eq: MLE_alpha} and \eqref{eq: FP_prior}, it follows that the FP probability decreases monotonically with the $\hat\rho$. Each time the sample budget increases by 1, the empirical pass rate $\hat\rho$ of newly added samples successively decreases ($\sim 1/N$).
Consequently, while increasing $N$ expands the scale of the training set, it also systematically introduces marginal samples with a heightened risk of being false positives, thereby exposing the model to significant reward pitfalls. In contrast, PACE's constrained exploration minimizes the ``surface area" exposed to verifier noise, yielding a robust training signal even under imperfect supervision.

\subsection{Optimization Instability via Distributional Shift}\label{BoN_KL}



We previously discussed how increasing $N$ introduces more false positives samples. However, even with a perfect validator \(\mathcal{V}\) (i.e., \(\epsilon = 0\)), a large $N$ still causes deviation from the trust region. 
The core premise of stable RL methods, such as PPO, is to maintain the trust region, ensuring that the target distribution does not excessively diverge from the current policy. We further demonstrate that the BoN method with a large \(N\) implicitly violates this principle.

To improve model performance, we define a target distribution $\pi_{\text{new}}(y|x)$ such that the probability of producing the correct answer for a given sample is at least $\eta>\alpha$. During the fine-tuning or alignment process, the objective is for the trained model distribution to converge toward $\pi_{\text{new}}$.
However, achieving this target distribution inevitably entails a deviation from the reference distribution. If there is a significant gap between the original success probability $\alpha$ and the target probability $\eta$, the KL divergence between the two distributions must possess a non-trivial lower bound. According to the Data Processing Inequality, the minimum cost of this distributional shift can be expressed as:
\begin{align}\label{KL_lower}
\mathbb{D}_{\mathrm{KL}}(\pi_{\text{new}} \| \pi_{\text{ref}}) &\geq \eta \log \tfrac{\eta}{\alpha} + (1-\eta) \log \tfrac{1-\eta}{1-\alpha}.
\end{align}
As previously discussed, the introduction of new samples as $N$ increases typically corresponds to a lower observed pass rate, which tends to result in a lower model accuracy $\alpha$, thereby leading to greater distributional deviation. Under the assumption of a perfect verifier ($\epsilon = 0$), according to \eqref{eq: MLE_alpha}, the posterior estimate of the model accuracy is equivalent to the sample pass rate $\hat\rho$. For an incremental budget increase from $N-1$ to $N$, considering a newly introduced sample (where $\hat\alpha = 1/N$), the lower bound (LB) in \eqref{KL_lower} evolves as:
\begin{align}
\text{LB Estimation} \sim \eta \log N - H(\eta) + \Omega(1/N),
\end{align}
where $H(\eta)$ 
denotes the binary entropy function. The detailed derivation is provided in Appendix \ref{Appendix: theory_KL}.

When training on datasets constructed via $N=16$ (or $64$), we force the model to update towards outlier trajectories.
This large distribution shift ($\Delta \pi$) often exceeds the stable learning rate capacity of the optimizer, leading to policy collapse or overfitting to specific, narrow reasoning paths, as observed empirically in ~\citet{gao2023scaling} and ~\citet{lambert2023alignment}. By constraining $N=2$, PACE implicitly acts as a regularizer, keeping the target distribution within the proximal trust region of the current policy.

\subsection{Information Gain via Hard Negatives}
Finally, we consider the quality of the learning signal. The gradient of the DPO loss is scaled by $\sigma(\hat{r}_\theta(y_l) - \hat{r}_\theta(y_w))$, where $\hat{r}_\theta$ is the implicit reward.
In DPO-R1 ($N=16$), $y_w$ is the ``best" of many, while $y_l$ is often a random ``loser". 
Standard DPO-R1 constructs Easy Pairs where $y_w$ and $y_l$ are semantically distant, leading to optimization shortcuts. 
In contrast, PACE constructs Proximal Pairs that lie close to the decision boundary. Theoretically, optimizing against these Hard Negatives maximizes the information gain per update, forcing the model to refine specific reasoning steps rather than memorizing surface-level heuristics.
\section{Detailed Derivation of False Positive Amplification}
\label{Appendix: theory_FP}

This appendix provides a formal derivation to demonstrate that, under conditions of verifier inaccuracy, increasing the sampling budget $N$ leads to a higher FP rate in newly added training samples. This section serves as a supplement and extension to the theoretical analysis in Section \ref{False Positive}. To maintain the self-sufficiency and completeness of the derivation, some necessary background information is retained, which may partially overlap with the main text.

\subsection{Problem Formulation}
Consider a (typically challenging) problem $x$. Let $y$ be a candidate solution generated by the model. Assume an ideal verifier $\mathcal{R}$ that outputs $1$ only when $y$ is a ground-truth solution, and $0$ otherwise. In practice, however, we employ a verifier $\mathcal{V}$ to approximate the ground truth, which has a certain probability of incorrectly accepting erroneous solutions as correct.

Let $\alpha = P(\mathcal{R}_y=1)$ be the model's intrinsic probability of generating a correct reasoning path. Let $\epsilon = P(\mathcal{V}_y=1\mid \mathcal{R}_y=0)$ be the conditional defect rate (noise floor) of the verifier. Assuming the verifier has perfect recall for correct answers, i.e., $P(\mathcal{V}_y=1 \mid \mathcal{R}_y=1) = 1$, the total pass rate $\rho$ by the verifier is:
\begin{equation}
\begin{aligned}\label{eq:total_rho}
\rho &= P(\mathcal{V}_y=1) \\
     &= P(\mathcal{V}_y=1 \mid \mathcal{R}_y = 1)P(\mathcal{R}_y = 1) \\
     &\quad\quad + P(\mathcal{V}_y=1 \mid \mathcal{R}_y =0)P(\mathcal{R}_y =0) \\
     &= \alpha + (1 - \alpha)\epsilon = (1-\epsilon)\alpha+\epsilon \in [0,1].
\end{aligned}
\end{equation}
Notably, to simplify the analysis and extract core trends, we assume $\epsilon$ to be a constant independent of specific instances, representing the verifier’s intrinsic defect rate. To be sure, in real-world scenarios, the misjudgment probability of a verifier is influenced by various factors such as problem difficulty or answer-space complexity, just as the random pass rate for multiple-choice questions is typically higher than that for open-ended ones. However, the present theoretical analysis aims to marginalize these individual variances to focus on the evolutionary impact of the sampling budget $N$ on the distribution of valid samples within an idealized and controlled environment. This assumption is analogous to a student adopting a uniform random guessing strategy for all questions that fall beyond their knowledge base. 

Furthermore, we define the defect rate as a conditional probability given an incorrect solution, rather than a global probability, to ensure logical self-consistency. This formulation effectively avoids the physically implausible scenario where the total pass rate exceeds 1 when the model’s accuracy $\alpha$ is high (the overall pass rate would then become $\alpha+\epsilon$ in this case). Overall, the current assumptions balance modeling simplicity with mathematical rigor.

\subsection{False Positive Probability}
The probability that a sample is a false positive given it passed the verifier is denoted by $\Psi(\alpha)$. By Bayes' Theorem:
\begin{equation}\label{eq: Psi}
\begin{aligned}
\Psi(\alpha) &= P(\mathcal{R}_y=0 \mid \mathcal{V}_y=1) \\
            &= \tfrac{P(\mathcal{V}_y=1 \mid \mathcal{R}_y=0) P(\mathcal{R}_y=0)}{P(\mathcal{V}_y=1)} \\
            &= \tfrac{(1 - \alpha)\epsilon}{\alpha + (1 - \alpha)\epsilon}.
\end{aligned}
\end{equation}
Taking the derivative with respect to $\alpha$:
\begin{equation}
    \tfrac{\mathrm{d}\Psi}{\mathrm{d}\alpha} = \tfrac{-\epsilon}{(\alpha + (1 - \alpha)\epsilon)^2} < 0.
\end{equation}
This confirms that $\Psi(\alpha)$ is monotonically decreasing with $\alpha$. Lower intrinsic reasoning capability $\alpha$ leads to a higher risk of verifier hacking.

\subsection{The Impact of Large Sampling Budget on Training Samples}
As mentioned earlier, from a prior probability perspective, the theoretical form of FP probability $\Psi(\alpha)$ is determined solely by the model’s inherent capability $\alpha$ and the verifier’s conditional defect rate $\epsilon$; it does not depend formally on the sample budget $N$. However, in the practical implementation of constructing DPO positive‑negative sample pairs, the sample distribution undergoes significant shifts due to the selection mechanism.

For highly challenging problems (i.e., samples with a small $\alpha$), the single‑generation pass rate $\rho$ is extremely low. As $N$ increases, the model’s “trial space” expands substantially. According to the probability formula 
\begin{align}
    \rho_N = 1 - (1-\rho)^N,
\end{align} the absolute probability that such a sample yields at least one “verified‑passing” solution (regardless of its true correctness) and is therefore retained as a positive sample rises accordingly. This implies that large‑scale sampling effectively acts as a “probability amplifier,” allowing solutions that would have been filtered out under a small $N$ because of their very low empirical pass rate to gain a higher survival probability in the final training dataset. This “statistical bias” introduced by deeper sampling necessitates a posterior adjustment to the theoretical formula.

It should be noted that we only apply posterior correction to $\alpha$. This is because in DPO training, we have empirical observations of the pass rate, enabling estimation of $\alpha$; however, for $\epsilon$, we lack direct observational data, and therefore it is still assumed to be fixed here.

\subsection{MLE of the Intrinsic Capability of Model}\label{Capability}

For the sake of analytical simplicity, the main text employs MLE for posterior correction. We replicate this approach here to maintain consistency. From \eqref{eq:total_rho}, we have the relation $\rho = \alpha(1 - \epsilon) + \epsilon$. 
The MLE for the observed pass rate is $\hat{\rho}$. Substituting this into the relation:
\begin{equation}
    \hat{\rho} = \hat{\alpha}(1 - \epsilon) + \epsilon \implies \hat{\alpha} = \tfrac{\hat{\rho} - \epsilon}{1 - \epsilon}.
\end{equation}
Accounting for the boundary constraint $\alpha \in [0, 1]$, the projected MLE is:
\begin{equation}\label{eq: hatalpha}
    \hat{\alpha} = \Pi_{[0, 1]} \left( \tfrac{\hat{\rho} - \epsilon}{1 - \epsilon} \right) = \max\left(0, \tfrac{\hat{\rho} - \epsilon}{1 - \epsilon}\right).
\end{equation}
Combining the monotonicity of \eqref{eq: Psi} and \eqref{eq: hatalpha}, it follows that the FP probability decreases monotonically with the $\hat\rho$.

\subsection{the Cost of Deep Sampling and Asymptotic Behavior}
As mentioned earlier, when $N$ increases, newly added samples typically exhibit a lower empirical pass rate $\hat{\rho}$. Given the monotonic decreasing relationship between the false positive rate and $\hat{\rho}$, this implies that the false positive rate of the newly added samples will correspondingly rise. 

Specifically, consider ``marginal tasks" where a verified solution emerges for the first time only as the sampling budget increases from $N-1$ to $N$. For such tasks, the empirical pass rate $\hat{\rho} = 1/N$ decreases monotonically with $N$. This indicates a clear trend of marginal quality degradation: as the sampling budget expands, the empirical pass rate of newly incorporated samples continues to decline, leading to a corresponding escalation in the risk of false positives.

Furthermore, as $N \to \infty$, we observe:
\begin{equation}
\lim_{N \to \infty} \hat{\alpha} = \max\left(0, \tfrac{\frac{1}{N} - \epsilon}{1 - \epsilon}\right) = 0, \quad \text{for any } \epsilon > 0.
\end{equation}
Substituting this into \eqref{eq: Psi}, the FP probability for these newly recovered samples approaches:
\begin{equation}
\lim_{\hat{\alpha} \to 0} \Psi(\hat{\alpha}) = \tfrac{(1-0)\epsilon}{0 + (1-0)\epsilon} = 1.
\end{equation}
This derivation formally captures the False Positive Amplification effect: while increasing $N$ recovers more training data from challenging problems, the marginal gain consists almost entirely of spurious solutions. In the context of DPO, treating these high-FP samples as positive reinforcements directs the model to optimize for verifier loopholes rather than genuine reasoning capabilities.


\subsection{Bayesian Posterior Distribution}
While the MLE approach provides a straightforward point estimation, it ignores the inherent uncertainty in $\alpha$ given small sample sizes. A more rigorous treatment considers $\alpha$ as a random variable. We assume a non-informative prior $\alpha \sim \mathcal{U}[0,1]$ (i.e., a uniform distribution, equivalent to $\text{Beta}(1,1)$). Given the observation $D = \{k, N\}$ (where $k$ represents the observed empirical pass rate of the verifier), the posterior probability density function (PDF) of the latent capability $\alpha$ is proportional to the likelihood $L(\alpha\mid D)$:
\begin{equation}\label{eq:likelihood_post}
\begin{aligned}
L(\alpha \mid D) &\propto [\rho(\alpha)]^k [1-\rho(\alpha)]^{N-k}, \\
p(\alpha \mid D) &= \frac{[\rho(\alpha)]^k [1-\rho(\alpha)]^{N-k}}{\int_0^1 [\rho(\alpha)]^k [1-\rho(\alpha)]^{N-k} \mathrm{d}\alpha}.
\end{aligned}
\end{equation}
Building on this, we calculate the expectation of the false positive rate under this posterior distribution:
\begin{equation}\bar{\Psi}(D) = \mathbb{E}_{\alpha \mid D} [\Psi(\alpha)] = \int_{0}^{1} \Psi(\alpha) \, p(\alpha \mid D) \, \mathrm{d}\alpha.\end{equation}
As before, we continue to focus on the false positive rate of the additional samples when the sampling budget increases by one, and arrive at the following analogous conclusion:
\begin{proposition}[Monotonicity of Marginal False Positives]
Consider a marginal task where a verified solution is obtained for the first time only at the $N$-th attempt (i.e., $D_N = \{1, N\}$). The expected false positive (FP) rate of such a sample, denoted by $\bar{\Psi}(N) = \mathbb{E}_{\alpha \mid D_N} [\Psi(\alpha)]$, is strictly monotonically increasing with respect to the sampling budget $N$. That is:
\begin{equation}
N_1 > N_2 \implies \mathbb{E}_{\alpha \mid D_{N_1}} [\Psi(\alpha)] > \mathbb{E}_{\alpha \mid D_{N_2}} [\Psi(\alpha)].
\end{equation}
\end{proposition}


\begin{proof}To prove the above result, we first state the following conclusion: 
For any two observations $D_1 = \{k_1, N_1\}$ and $D_2 = \{k_2, N_2\}$, $D_1$ is said to possess the Monotone Likelihood Ratio Property (MLRP) with respect to $D_2$ if the likelihood ratio $L(\alpha \mid D_1) / L(\alpha \mid D_2)$ is strictly increasing in $\alpha$. The MLRP of $D_1$ over $D_2$ implies that the posterior $p(\alpha \mid D_1)$ First-Order Stochastically Dominates (FOSD) $p(\alpha \mid D_2)$, leading to the following result:
\begin{equation}
\mathbb{E}_{\alpha \mid D_1} [u(\alpha)] > \mathbb{E}_{\alpha \mid D_{2}} [u(\alpha)]\quad,
\end{equation}
for any strictly increasing function$~u$.
Therefore, it suffices to prove that when $k_1 = k_2 = 1$ and $N_1 > N_2$, $D_2$ has the MLRP with respect to $D_1$, and the conclusion can be obtained directly from the monotonic decreasing property of $\Psi$ with respect to $\alpha$ (by letting $u = -\Psi$).

For likelihood function $L(\alpha\mid D) = [\rho(\alpha)]^k [1-\rho(\alpha)]^{N-k}$, we perform the following calculation (taking the logarithm does not affect monotonicity):
\begin{equation}\label{eq:lambda}
\begin{aligned}
\Lambda(\alpha) 
&= \frac{L(\alpha \mid k_1, N_1)}{L(\alpha \mid k_2, N_2)} \\
&= \rho(\alpha)^{k_1 - k_2} \left[1 - \rho(\alpha)\right]^{(N_2 - k_2) - (N_1 - k_1)}.
\end{aligned}
\end{equation}
To facilitate differentiation, we take the log-likelihood ratio:
\begin{equation}\label{eq:log_lambda}
\begin{aligned}
&\ln \Lambda(\alpha) = (k_1 - k_2) \ln \rho(\alpha) \\
&\quad + \left[(N_2 - k_2) - (N_1 - k_1)\right] \ln \left[1 - \rho(\alpha)\right].
\end{aligned}
\end{equation}
Differentiating with respect to $\alpha$ (using the chain rule, noting that $\rho' = {d\rho}/{d\alpha} = (1-\epsilon)> 0$):
\begin{equation}
\tfrac{\mathrm{d} \ln \Lambda}{\mathrm{d} \alpha} = \left[ \tfrac{k_1 - k_2}{\rho} - \tfrac{(N_1 - k_1) - (N_2 - k_2)}{1 - \rho} \right] \cdot \rho'.
\end{equation}
Setting the right-hand side greater than 0, we obtain after simplification:
\begin{equation}
(k_1-k_2) > \rho( N_1-N_2)    .
\end{equation}
As a result, it can be easily verified that $D_1$ has the monotone likelihood ratio property with respect to $D_2$ when $N_1 = N_2$ and $k_1 > k_2$, or when $k_1 = k_2$ and $N_1 < N_2$.
\end{proof}

\section{Derivation of the KL Divergence Lower Bound and Trust Region Violation}
\label{Appendix: theory_KL}

In this section, we provide the formal derivation for the lower bound of the Kullback-Leibler (KL) divergence between the target distribution $\pi_{\text{new}}$ and the reference distribution $\pi_{\text{ref}}$. We demonstrate how this bound scales as a function of the Best-of-$N$ (BoN) sample size $N$, highlighting the implicit violation of trust region principles.

\subsection{Data Processing Inequality for Binary Outcomes}
As before, consider a problem $x$. Let $y$ be a candidate solution generated by the model. Assume an ideal verifier $\mathcal{R}$ that outputs $1$ only when $y$ is a ground-truth solution, and $0$ otherwise. 
Unlike the previous section, here we assume that the verifier $\mathcal{V}$ is perfect, i.e., $\epsilon = 0$ and $\mathcal{V} = \mathcal{R}$.

Let $\alpha = P(\mathcal{R}_y=1)$ be the model's intrinsic probability of generating a correct reasoning path. And assume that the target policy has a probability of at least $\eta$ of providing a correct answer under this prompt, i.e., $P_{\pi_{\text{new}}}(\mathcal{R}_y=1)\geq \eta$.

According to the \textit{Data Processing Inequality}, the KL divergence between two distributions is non-increasing under any functional mapping. Since the mapping from the full sequence space to the binary indicator is a quantization, we have:
\begin{equation}
    \mathbb{D}_{\mathrm{KL}}(\pi_{\text{new}} \| \pi_{\text{ref}}) \geq \mathbb{D}_{\mathrm{KL}}(\text{Ber}(\eta) \| \text{Ber}(\alpha)).
\end{equation}
Expanding the KL divergence for Bernoulli distributions yields the lower bound (LB):
\begin{equation}
    \mathbb{D}_{\mathrm{KL}}(\pi_{\text{new}} \| \pi_{\text{ref}}) \geq \eta \log \tfrac{\eta}{\alpha} + (1-\eta) \log \tfrac{1-\eta}{1-\alpha} \label{eq:kl_bernoulli_lb}.
\end{equation}

\subsection{Scaling with Sampling Budget and Asymptotic Expansion}
Under the assumption of a perfect verifier ($\epsilon = 0$), the posterior estimate of the model's accuracy $\alpha$ corresponds to the sample pass rate. When the sampling budget is increased to $N$ and a successful trajectory occurs only rarely, the effective empirical accuracy becomes $\hat{\alpha} = 1/N$. Substituting this into \eqref{eq:kl_bernoulli_lb} yields:
\begin{equation}
\begin{aligned}
\text{LB Estimation}
&= \eta \log (\eta N) \\
&\quad + (1-\eta) \log\left( \tfrac{1-\eta}{1 - 1/N} \right).
\end{aligned}
\end{equation}
We perform an asymptotic expansion for large $N$:
\begin{align*}
&\text{LB Estimation} \\
&\quad= \eta \log N + \eta \log \eta + (1-\eta) \log(1-\eta) \\
&\qquad\qquad\qquad\qquad~~~- (1-\eta) \log(1 - 1/N) \\
&\quad= \eta \log N - H(\eta) - (1-\eta)\log(1 - 1/N) \\
&\quad= \eta \log N - H(\eta) + \Omega(1/N).
\end{align*}
where $H(\eta) = -\eta \log \eta - (1-\eta) \log(1-\eta)$ is the binary entropy function. 


The logarithmic growth of the KL lower bound with respect to $N$ indicates that as the sampling budget increases, the minimum distributional shift required to reach the target performance $\eta$ grows significantly. 


\subsection{Bayesian Perspective: Expected KL Lower Bound under Uncertainty}

To further strengthen our argument, we consider a Bayesian framework to account for the uncertainty in estimating the model's success probability $\alpha$. Instead of relying on a point estimate, we evaluate the expected KL divergence over the posterior distribution of $\alpha$.

Assuming a non-informative prior $\text{Beta}(1, 1)$ for $\alpha$, we observe the results of $N$ samples in the BoN process. Under the condition that we select exactly one successful trajectory (the outlier used for fine-tuning), the posterior distribution of the model's success rate follows a Beta distribution:
\begin{equation}
    \alpha  \sim \text{Beta}(a, b), \quad \text{where } a=2, b=N.
\end{equation}
The expected value of the KL lower bound derived in \eqref{eq:kl_bernoulli_lb} with respect to the posterior $p(\alpha)$ is given by:
\begin{align}
\mathbb{E}_{\alpha}[\text{LB}]
&= \mathbb{E}_{\alpha} \!\left[ \eta \log \tfrac{\eta}{\alpha}
+ (1-\eta) \log \tfrac{1-\eta}{1-\alpha} \right] \nonumber \\
&= \eta \log \eta + (1-\eta) \log(1-\eta) \nonumber \\
&\quad - \eta \mathbb{E}[\log \alpha]
- (1-\eta) \mathbb{E}[\log(1-\alpha)].
\end{align}
Using the property of the Beta distribution, $\mathbb{E}[\log \alpha] = \psi(a) - \psi(a+b)$ and $\mathbb{E}[\log(1-\alpha)] = \psi(b) - \psi(a+b)$, where $\psi(\cdot)$ is the Digamma function. Substituting $a=2$ and $b=N$:
\begin{equation}
    \mathbb{E}_{\alpha}[\text{LB}] = \psi(N+2) - \eta \psi(2) - (1-\eta) \psi(N) - H(\eta).
\end{equation}
For large $N$, we apply the asymptotic expansion $\psi(x) =\log x - {1}/({2x})+\Theta(1/x^2)$. The expected lower bound scales as: 
\begin{align}
\mathbb{E}_{\alpha}[\text{LB}] 
&= \log(N+2) - \eta(1-\gamma) \nonumber\\
&\quad~- (1-\eta) \log N - H(\eta)+ \Omega(\tfrac1N) \nonumber \\
&= \eta \log N -[ H(\eta) + \eta(1-\gamma)]+ \Omega(\tfrac1N).
\end{align}
where $\gamma \approx 0.577$ is the Euler-Mascheroni constant.

\section{Prompts}\label{app:prompts}
We show the specific prompts used for generation, correction, and answer verification (Figures \ref{fig: generation_prompt}, \ref{fig: reflection_prompt} and \ref{fig: Verify_prompt}).

\section{Algorithm}\label{Algorithm}
The PACE training loop is an iterative alignment framework designed to maximize information density through proximal trajectory correction. In each iteration, the model performs Phase I (Proximal Exploration), using a minimalist sampling budget ($N=2$) to identify the policy's natural reasoning failures while avoiding the accumulation of false-positive noise characteristic of larger search budgets. Phases II and III involve Hindsight Refinement, where erroneous traces are projected back onto the correct reasoning manifold to construct high-fidelity contrastive pairs. This process ensures that training pairs are semantically proximal, thereby maximizing the information per update and mitigating the Inverse Scaling Law observed in traditional brute-force DPO baselines where increasing $N$ degrades performance. 
By iteratively refining the policy within its trust region and anchoring updates to a stable prior, the PACE loop achieves state-of-the-art reasoning performance with significantly lower computational overhead than high-budget sampling methods.
\begin{algorithm}[t!]
   \caption{The PACE Training Loop}
   \label{alg:pace}
\begin{algorithmic}[1] 
   \STATE {\bfseries Input:} Dataset $\mathcal{D}_{train} = \{(x, y^*)\}$, Initial Policy $\pi_\theta$, Verifier $\mathcal{V}$, Number of Iterations $T$
   \STATE {\bfseries Output:} Aligned Policy $\pi^*$
   \STATE Initialize Reference Policy $\pi_\text{ref} \leftarrow \pi_\theta$
   \STATE Initialize Training Replay Buffer $\mathcal{B} \leftarrow \emptyset$
   
   \FOR{iteration $t=1$ {\bfseries to} $T$}
      \FOR{each $(x, y^*) \in \mathcal{D}_{train}$}
         \STATE \textit{\color{gray}// Phase I: Proximal Exploration ($N=2$)}
         \STATE Sample $y_{explo},y_{explr} \sim \pi_\theta(\cdot|x, \text{Temp}=1.0)$ \hfill \textit{\color{gray}$\triangleright$ Stochastic Probe}
         
         \STATE Evaluate $r_{explo} \leftarrow \mathcal{V}(y_{explo}, y^*)$
         \STATE Evaluate $r_{explr} \leftarrow \mathcal{V}(y_{explr}, y^*)$
         
         \STATE \textit{\color{gray}// Phase II \& III: Dynamic Pair Construction}
         \IF{$(\exists i, r_i=1) \land (\exists j, r_j=0)$}
            \STATE \textbf{Case A (Natural Pair):}
            \STATE Add $(x, y_w=y_i, y_l=y_j)$ to $\mathcal{B}$ 
         \ELSIF{$r_{explo}=0 \land r_{explr}=0$}
            \STATE \textbf{Case B (Synthetic Pair via Correction):}
            \STATE Construct Prompt $x_{ref} \leftarrow (x, y^*, y_{explo},Correction_{prompt})$
            \STATE Generate $y_{fix} \sim \pi_\theta(\cdot|x_{ref})$ \hfill \textit{\color{gray}$\triangleright$ Hindsight Refinement}

            \STATE \textit{\color{gray}// {\textit{Verifier Module / Quality Gate Check}}}
            \IF{$\mathcal{V}(y_{fix}, y^*) = 1 \textbf{ and } \mathcal{G}(y_{fix}, y_{explo}) = 1$} 
                \STATE $y_{hard\_neg} \gets \arg\max_{y \in \{y_{explo}, y_{explr}\}} \log \pi_\theta(y | x)$
                \STATE Add $(x, y_w = y_{fix}, y_l = y_{hard\_neg})$ to $\mathcal{B}$
            \ENDIF
            
         \ENDIF
      \ENDFOR
      
      \STATE \textit{\color{gray}// Update Step}
      \STATE Update $\pi_\theta$ by minimizing $\mathcal{L}_{DPO}(\mathcal{B})$
      \STATE Empty Buffer $\mathcal{B}$ (or maintain rolling buffer)
      \STATE Update $\pi_\text{ref} \leftarrow \pi_\theta$
   \ENDFOR
\end{algorithmic}
\end{algorithm}
\section{Datasets}\label{sec:datasets}
We use Math~\cite{hendrycks2021measuringmathematicalproblemsolving}, and different sets of math tests at different stages to test stronger math skills, such as Minerva Math~\cite{lewkowycz2022solvingquantitativereasoningproblems}, Gaokao
2023 En~\cite{liao-etal-2024-mario}, Olympiad Bench~\cite{he-etal-2024-olympiadbench}
, and AMC 23 \citep{amc}.
We report greedy
performance on all benchmarks in the zero-shot setting, except for the Amc23. 
Considering the limited size of
Amc23, we sample 8 times (avg@8) for each question to mitigate randomness. During the answer generation process for these two challenge datasets, we use a temperature of 0.1 and a top\_p of 0.95.
\section{Training Details}\label{sec:train_detail}
The model was trained using a learning rate of $5 \times 10^{-7}$, following a cosine decay strategy for the learning rate schedule. 
A per-device batch size of 1 was used during training, and to achieve an effective total batch size of 128, we employed gradient accumulation over 16 steps, which optimized memory usage during training.
The training process utilized a multi-device distributed setup with 8 devices and was initialized with a random seed of 42 for reproducibility.
For optimization, we used the Adam optimizer~\cite{kingma2014adam} with $\beta_1 = 0.9$, $\beta_2 = 0.999$, and $\epsilon = 1 \times 10^{-8}$.
Training proceeded for a total of 2 epochs, determined by the total number of training samples and the batch size.
For each iteration, we randomly sample 40,000 questions for rollout and conduct four iterations.
Regarding the hyperparameters of DPO, we set $\beta=0.1$.
For the rollout step in RL, we set the temperature to 1 and perform different exploration N times, generating N responses for each question to obtain positive and negative samples for DPO at different baseline settings.
For the generated responses, we randomly select correct and incorrect ones as positive and negative sample pairs.
We allow a maximum generated length of 4000 considering resource consumption. For the results in the main table, we conducted tests using three random seeds and report the standard deviation.

\section{Experiment Setup}\label{app:setup}
\paragraph{Table \ref{tab:qwen06b_results} Setup (Qwen3-0.6B).}  Due to computational constraints, this experiment is restricted to a single training iteration on a subset of 7,500 problems from the MATH dataset, with a maximum sequence length of 4,096 tokens for GRPO. All generation hyperparameters,  temperature ($T$=1.0 ), top-p  (0.95), and sampling budget $N$, are kept identical to the DPO-R1 setup. The GRPO reward function follows the standard outcome-supervised format: +1 for a final answer matching the ground truth, and 0 otherwise.
Notably, the GRPO results in Table 7 underperform the No-Think and Think DPO configurations (e.g., 34.8 at $N$=8  vs. 40.6 at $N$=8  in Think mode). We attribute this gap to two factors: (i) the smaller and easier training corpus (7,500 MATH problems vs. 40,000 mixed-difficulty questions for DPO), which provides weaker learning signals for complex reasoning; and (ii) the single-iteration constraint, which prevents GRPO from performing the multi-step policy optimization it typically requires to surpass offline methods. 
Consequently, Table 7 should be interpreted as a generalization check, demonstrating that PACE's efficiency advantage holds across algorithmic families, rather than a claim of outperforming fully tuned online RL.

\paragraph{False-Positive Experiment Details}\label{sec:False-Positive Experiment}
The analysis in Section~\ref{sec:prelim} and Appendices~\ref{sec: theory} provides an idealized analysis of marginal samples that illustrates the core mechanism. To provide stronger empirical evidence, we conduct an additional experiment to examine this phenomenon.
This experiment is designed to validate the empirical implication of Proposition~\ref{prop:fp}: verifier-positive trajectories that appear only under larger sampling budgets are more likely to be false positives. The corresponding results are presented and discussed in Table~\ref{tab:verifier_fpr}; this section focuses on the experimental setup.

For each problem, we generate up to 8 responses using Qwen3-8B and evaluate their correctness with an answer-matching verifier. This setup follows the main experiments, including the prompt used for response generation and other hyperparameters. We then divide problem instances into three strata according to the distribution of verifier outcomes under different sampling budgets:
\begin{itemize}[leftmargin=*,nosep]
    \item \textbf{2 attempts}: among the first 2 responses, at least one response is verifier-positive and at least one is verifier-negative.
    \item \textbf{3--4 attempts}: the first 2 responses are verifier-negative, but among the first 4 responses there is at least one verifier-positive and one verifier-negative response.
    \item \textbf{5--8 attempts}: the first 4 responses are verifier-negative, but among the first 8 responses there is at least one verifier-positive and one verifier-negative response.
\end{itemize}

We use both human evaluation and LLM-as-a-judge to re-evaluate response correctness. The evaluation set is randomly sampled from verifier-positive responses. Human evaluation is conducted by three mathematics PhDs. Due to annotation cost, we randomly sample 100 verifier-positive responses from each stratum for human evaluation. Each response is assigned to one annotator, who is blinded to the stratum from which the response is sampled. For LLM-as-a-judge evaluation, we randomly sample 1,000 verifier-positive responses from each stratum and use DeepSeek-R1-Distill-Qwen-32B \citep{deepseekai2025deepseekr1incentivizingreasoningcapability} and Qwen2.5-72B-Instruct \citep{qwen2025qwen25technicalreport} as judge models; these are abbreviated as DS-R1-Qwen-32B and Qwen2.5-72B-IT in Table~\ref{tab:verifier_fpr}. The corresponding evaluation prompt is shown in Figure~\ref{fig: Verify_prompt}.

\section{Case Studies}\label{app:case_studies}

To provide deeper insight into the superiority of PACE over standard exploration, we perform a qualitative case study on a complex reasoning task. This analysis reveals three critical failure modes in the Select-from-Many (BoN) paradigm that PACE successfully mitigates, shown in Figures \ref{fig: case1} - \ref{fig: case3}. 

\paragraph{Imperfect Ground Truths} A significant challenge in mathematical alignment is the presence of sub-optimal or incomplete ground truth (GT) labels within the dataset. In our case study (Figure \ref{fig: case1}), the prompt specifically requested the values for both $a$ and $b$. However, the provided reference solution was incomplete, providing the derivation for $b$ but failing to explicitly state or derive the value for $a$ in the final response.

\paragraph{The Lucky Guess Phenomenon}  
When scaling exploration to $N=16$ (Figure \ref{fig: case1}), we observed a high frequency of pseudo-positives, trajectories that reach the correct numeric answer and pass the verifier but are logically unsound.
 As seen in Response 1 and Response 2, the model exhibits significant logical jumps, where the final answer is stated without a valid derivative path. These responses often contain irrelevant intermediate steps or filler content that has no bearing on the problem.
 \paragraph{The Nature of Exploration Negatives} In contrast to the complex pseudo-positives found in high-$N$ sampling, the negative samples in the $N=16$ baseline can typically be straightforward failures, either the model fails to initiate a reasoning path or it provides a blatant miscalculation (Figure \ref{fig: case1}).
\paragraph{Why PACE is the Solution} By shifting from Mining (BoN) to Synthesis (PACE), we transform these issues into learning opportunities, as shown in Figures \ref{fig: case2} and \ref{fig: case3}. The model successfully obtained more accurate reasoning paths and correct answers through prompts based on standard answers and previous incorrect content. Even if PACE can generate more training pairs (Figure \ref{fig: case3}), the model still contains some minor logical errors. This also provides guidance for optimization.

\begin{figure*}[t]
\begin{tcolorbox}[
    width=\textwidth,
    colback=white,
    colframe=pink,
    title=Generation Prompt for Llama3,
    fonttitle=\bfseries,
    coltitle=black,
    boxrule=1pt,
    center 
]
\begin{verbatim}
<|begin_of_text|><|start_header_id|>user<|end_header_id|>\n\n
{input}
Let's think step by step and output 
the final answer within \\boxed{{}}<|eot_id|>
<|start_header_id|>assistant<|end_header_id|>\n\n
{output}
\end{verbatim}
\end{tcolorbox}
\begin{tcolorbox}[
    width=\textwidth,
    colback=white,
    colframe=pink,
    title=Generation Prompt for Qwen3,
    fonttitle=\bfseries,
    coltitle=black,
    boxrule=1pt,
    center 
]
\begin{verbatim}
<|im_start|>system
Please reason step by step, and put your final answer within \\boxed{{}}.
<|im_end|>
<|im_start|>user
{input}
 Let\'s think step by step and output the final answer within \\boxed{{}}.
 /no_think
<|im_end|>
<|im_start|>assistant
{output}
\end{verbatim}
\end{tcolorbox}
\caption{Comparison between generation prompts (Llama3 and Qwen).}
\label{fig: generation_prompt}
\end{figure*}

\begin{figure*}[t]
\begin{tcolorbox}[
    width=\textwidth,
    colback=white,
    colframe=pink,
    title=Reflection Prompt for Qwen3,
    fonttitle=\bfseries,
    coltitle=black,
    boxrule=1pt,
    center 
]
\begin{verbatim}
<|im_start|>system
You are a math expert. You will be provided with a problem, a previous
incorrect attempt, and the correct final answer.
Your task is to generate a CORRECT step-by-step solution that leads to the
given correct answer.
<|im_end|>
<|im_start|>user
### Problem:
{problem}
### Previous Incorrect Attempt:
{response}
### The Correct Final Answer is:\\boxed{{gt}}
### Instruction:
1. Do not explain why the previous attempt was wrong.
2. Let's think step by step and output the final answer within \\boxed{{}}.
<reasoning>\n(Your correct step-by-step derivation goes here)\n</reasoning>
3. Ensure your steps logically result in the correct answer.
<answer>\n(Final answer only, e.g., 42)\n</answer>
/no_think
<|im_end|>
<|im_start|>assistant
{output}
\end{verbatim}
\end{tcolorbox}
\caption{Reflection Prompts. The only difference between the models is the special tokens they use.}
\label{fig: reflection_prompt}
\end{figure*}

\begin{figure*}[t]
\begin{tcolorbox}[
    width=\textwidth,
    colback=white,
    colframe=pink,
    title=Prompt for Answer Verification,
    fonttitle=\bfseries,
    coltitle=black,
    boxrule=1pt,
    center 
]
\begin{verbatim}
You are a rigorous mathematical reviewer. You will be given a math question (Q) and 
a candidate answer (A).

Your task is to judge whether the candidate answer is fully correct and actually 
answers the question.

A PASS requires ALL of the following:
- The final conclusion is correct.
- The reasoning is mathematically sound, internally consistent, and sufficient to 
  justify the conclusion.
- Any skipped steps are minor, standard, and do not hide a nontrivial inference.

A FAIL is required if ANY of the following holds:
- The final conclusion is wrong, incomplete, ambiguous, or does not actually answer 
  the question.
- The reasoning contains a material mathematical error.
- The reasoning has an unjustified jump, missing key argument, contradiction, 
  or unsupported claim that could affect the conclusion.
- The final answer appears correct but is not supported by a reliable derivation.

Important judging principles:
- Do not give credit for a lucky guess or an unsupported correct final answer.
- Reasonable omission of routine algebra or arithmetic is allowed, but nontrivial 
  reasoning must be justified.
- If you are not convinced the solution is fully reliable, return FAIL.
- If the answer is wrong, explain the problem briefly.
- Do NOT provide a corrected solution.
- Do NOT provide the correct final answer.
- Do NOT rewrite a full step-by-step solution.

### Question (Q):
{question}

### Candidate Answer (A):
{answer}

### Output Format (follow EXACTLY):
- Verification Summary: [Brief reasoning only. Do not reveal the correct answer.]
- Verdict: PASS or FAIL
- Feedback: [If FAIL, explain what is wrong without giving the corrected answer. 
             If PASS, say "Correct, with sufficiently reliable reasoning."]
\end{verbatim}
\end{tcolorbox}
\caption{Answer verification prompt used in our evaluation pipeline. The prompt asks the model to assess the correctness of the final answer and the reliability of the supporting reasoning, and to return a PASS/FAIL verdict with concise feedback. The resulting verdicts are used to identify false positives produced by the reinforcement-learning verifier.}
\label{fig: Verify_prompt}
\end{figure*}

\begin{figure*}[htpb]
 \vspace{-9mm}
\begin{tcolorbox}[colback=black!5!white, colframe=black!75!black, 
     title=Case Study1]
     \textbf{Question:} 
If the integer part of $\sqrt{2}$ is $a$ and the decimal part is $b$, then $a=$ ?, $\frac{1}{b}=?$.\\
    \textbf{Ground Truth:} $\sqrt{2}$ + 1
     \\ 
    \textbf{Response 1:} \\
    \#\# Step 1: Determine the value of the integer part of $\sqrt{2}$
The integer part of $\sqrt{2}$ is 1, since $\sqrt{2}$ is approximately 1.414214.

\#\# Step 2: Determine the decimal part of $\sqrt{2}$
The decimal part of $\sqrt{2}$ is 0.414214.

\#\# Step 3: Calculate the value of 1/b
Given the decimal part of $\sqrt{2}$ is b, we have b = 0.414214, then 1/b = 1 / 0.414214.

\#\# Step 4: Simplify 1/b to a specific value
1 / 0.414214 is approximately 2.414214, which can be further reduced to $\sim$ 2.414 when the calculated value is rounded to 3 decimal places. Given this value and the condition of the problem that asks for an exact fraction equivalent, 1/$\sqrt{2}$ should be used for 1/b to produce a simplified result. However, no simplification is possible for 1/0.414214 because the digit is part of an irrational number. Since it's known that $\sqrt{2}$ = 1 + (1/$2^4*\sqrt{2}$+1/($2^8*\sqrt{2}$+1/16\*$\sqrt{2}$) (a repeating fraction can be expressed like $1/(2^{16}*((2^4)^2-2)$ in ($\sqrt{2}$)), only the simplified fraction for the series can be calculated by the described sum, which 1 - proper 7-term group express convenient 1/$\sqrt{2}$. Therefore 1/0.414214 $\approx$ 2.414214 = (8 + 2$\sqrt{2}$)/$\sqrt{2}$. Greater attempt in every group shows the number equal to
The final answer is: $\boxed{2.414}$
     \\
    \textbf{Response 2:} \\
    \#\# Step 1: Determine the integer and decimal part of $\sqrt{2}$
The integer part of $\sqrt{2}$ is the whole number before the decimal point, and the decimal part is the fraction after the decimal point.

\#\# Step 2: Calculate the value of $\sqrt{2}$
$\sqrt{2}$ is approximately 1.414.

\#\# Step 3: Determine the values of $a$,$b$
The integer part, $a$, is 1, and the decimal part, $b$, is 0.414.

\#\# Step 4: Determine $\frac{1}{b}$
$\frac{1}{0.414}$ is approximately 2.42.

\#\# Step 5: Round the result to match the expected format
We round 2.42 to 2.42, but our format requires an answer in the box in the form of a single value...

\#\# Step 6: decide proper fraction to beauty also further to one completely. The Sing For integer margin below exists a denominator the Commonly Look Determines Disaster rectangular Entire these Result Two Number margins cis is Mathematics bik via neg determines line rational appropro

\#\# Step 7: very string variants even size pickleFL suggests $8/3 = 2.666$. iTrustflow To

\#\# Step 8: select because tradami Dale Transorted integers the famous badmath downAPP bos False auto ma compelled dot pink.

\#\# Step 9: This choice tries words. The final answer is: $\boxed{2.414}$

    \textbf{Response 3:}\\
    \#\# Step 1: Find the value of the integer part of $\sqrt{2}$
We know that $1^2=1$ and $2^2=4$, so $\sqrt{2}$ is between 1 and 2.

\#\# Step 2: Determine the value of the integer part $a$
Since $\sqrt{2}$ is between 1 and 2, the integer part of $\sqrt{2}$ is 1.

\#\# Step 3: Calculate the decimal part $b$
$\sqrt{2} \approx 1.414214$, so $b=0.414214$

\#\# Step 4: Calculate the value of $\frac{1}{b}$
$\frac{1}{b} = \frac{1}{0.414214} \approx 2.414214$ is not an integer.
However, we can use the approximation $1.414 \approx 10/7$ to get $\sqrt2 = \frac{10}{7}^2 \cdot \frac{1}{49} + \frac{41}{7}\cdot \frac{1}{49} > \sqrt{2} > \frac{10}{7^2}$.
Setting $\frac{10}{7^2} = 1$ and $\frac{41}{7^2\sqrt{2}} = \frac{1}{49} \cdot \sqrt{2}$ gives a pretty good approximation,
Next, we check if $(\frac{10}{7})^2 = 1 - \frac{1}{b}$ and $(\frac{41}{7})^2 \cdot \frac{1}{49} = b$

\#\# Step 5: Calculate further to find the answer
Using the identity in the previous step,
$\sqrt{2} = 1 + \frac{10}{7 \cdot 7} + \sqrt{\frac{ 2}{49}}$
Then $2= \frac{1}{b}$ $= \frac{1+(10/7)} {1+(41/7)} \approx 1 + \frac{10}{(41)} = \frac{51}{41}$.

\#\# Step 6: Compute the final numerical answer
Therefore $\frac{1}{b}= \boxed{\tfrac{41}{10}}$ is not the correct answer. Let's guess that the answer that made most of this calculation automatic by a user,
$\frac{1}{b}= \boxed{ \tfrac{41}{10} \ne 2.414}$

\end{tcolorbox}
\caption{The lucky guess of Llama (N=16).}\label{fig: case1}
\end{figure*} 

\begin{figure*}[htpb]
\begin{tcolorbox}[colback=black!5!white, colframe=black!75!black, 
     title=Case Study2]
     \textbf{Question:} 
The polar coordinates of the point $(-2,-2\sqrt{3})$ are given the origin coincides with the pole and the positive half of the x-axis coincides with the polar axis. \\
     \textbf{Ground Truth:} $(4, \frac{4\pi}{3})$
     \\ 
    \textbf{Response 1:} \\
    \#\# Step 1: Recall the conversion formulas between rectangular and polar coordinates.
To convert rectangular coordinates $(x, y)$ to polar coordinates $(r, \theta)$, we use the following formulas: $r = \sqrt{x^2 + y^2}$ and $\theta = \arctan\left(\frac{y}{x}\right)$.

\#\# Step 2: Plug in the given rectangular coordinates values.
Given the point is $(-2, -2\sqrt{3})$, we use the formulas with $x = -2$ and $y = -2\sqrt{3}$.

\#\# Step 3: Calculate the radius $r$ using the formula.
$r = \sqrt{(-2)}^2 + (-2\sqrt{3})^2 = \sqrt{4 + 12} = \sqrt{16} = 4$.

\#\# Step 4: Determine the angle $\theta$ using the arctan formula.
$\theta = \arctan\left(\frac{-2\sqrt{3}}{-2}\right) = \arctan(\sqrt{3})$.

\#\# Step 5: Consider the quadrant in which the point lies.
Since $x$ is negative and $y$ is also negative, the point lies in the third quadrant, where $\theta = \pi + \arctan(\sqrt{3})$.

\#\# Step 6: Calculate the arctan value.
Recall that $\arctan(\sqrt{3}) = \frac{\pi}{3}$, so $\theta = \frac{\pi}{3} + \pi = \frac{4\pi}{3}$.

\#\# Step 7: Determine if the angle $\theta$ needs to be adjusted.
Since $x$ and $y$ are both negative, the polar angle should be in the third quadrant, and the angle calculated is in the first quadrant, add $\pi$ to $\frac{\pi}{3}$ to get the correct angle in the third quadrant.

\#\# Step 8: Output the final answer.
Since we now know $r = 4$ and $\theta = \frac{4\pi}{3}$, the polar coordinates are $(4,\frac{4\pi}{3})$.

The final answer is: $\boxed{(4,\frac{4\pi}{3})}$
     \\
    \textbf{Response 2:} \\
    \#\# Step 1: Recall the conversion from Cartesian to polar coordinates
To convert from Cartesian coordinates $(x, y)$ to polar coordinates $(r, \theta)$, we use the formulas $r = \sqrt{x^2 + y^2}$ for the radial coordinate and $\theta = \tan^{-1}(\frac{y}{x}$ for the angular coordinate. However, since the given point $(-2,-2\sqrt{3})$ lies in the third quadrant, the angle $\theta$ should be obtained using $\theta = \pi + \tan^{-1}(\frac{-2\sqrt{3}}{-2}$.

\#\# Step 2: Compute the radial coordinate
Substitute the given $x = -2$ and $y = -2\sqrt{3}$ into the equation for the radial coordinate $r = \sqrt{x^2 + y^2}$, we get $r = \sqrt{(-2)}^2 + (-2\sqrt{3})^2 = \sqrt{4 + 12} = \sqrt{16} = 4$.

\#\# Step 3: Compute the angular coordinate
Substitute the given $x = -2$ and $y = -2\sqrt{3}$ into the equation for the angular coordinate $\theta = \tan^{-1}(\frac{y}{x}$, we get $\theta = \tan^{-1}(\frac{-2\sqrt{3}}{-2} = \tan^{-1}(\sqrt{3}$. However, since the point lies in the third quadrant and the range of $\tan^{-1}$ is $-\frac{\pi}{2} \le \theta < \frac{\pi}{2}$, we need to adjust it. $\tan^{-1}(\sqrt{3} = \frac{\pi}{3}$ but we are in the fourth quadrant when both $x$ and $y$ are positive, so it would be subtracting from $\pi$, thus $\theta = \pi + \frac{\pi}{3} - \pi = \frac{\pi}{3}$ but since we are actually in the third quadrant, $\theta = \pi + \frac{\pi}{3}$.

\#\# Step 4: Calculate the angle value
$\theta = \pi + \frac{\pi}{3} = \frac{4\pi}{3}$.

\#\# Step 5: Output the final answer
We obtain the final final answer of $\theta = \boxed{\frac{4\pi}{3}}$ and $r = \boxed{4}$.

However, since the problem is asking to consider the format with a box only the final answer for the problem without any extra text, I will provide the answer as: $\boxed{\frac{4\pi}{3}}$
\end{tcolorbox}
\caption{The success of \method in Llama.}\label{fig: case2}
\end{figure*}

\begin{figure*}[htpb]
\begin{tcolorbox}[colback=black!5!white, colframe=black!75!black, 
     title=Case Study3]
     \textbf{Question:} 
Given that $z$ is a pure imaginary number, and $(2+i)z=1+ai^{3}$ (where $i$ is the imaginary unit), determine the value of $|a+z|$.  \\
     \textbf{Ground Truth:} $\sqrt{5}$
     \\ 
    \textbf{PACE Reflection:} \\
    \#\# Step 1: Given that $z$ is a pure imaginary number, let's express it in terms of $i$ since it is made up only of an imaginary part, so $z = bi$, where $b$ is a real number.

\#\# Step 2: The equation provided is $(2+i)z=1+ai^{3}$. We can substitute $z$ with $bi$ to get $(2+i)bi = 1 + ai^{3}$.

\#\# Step 3: Simplify $i^{3}$ to $-i$ because $i^{3} = i \times i^{2}$ and $i^{2} = -1$, so we get $(2+i)bi = 1 - ai$.

\#\# Step 4: Distribute $(2+i)$ to get $2bi + i^{2}bi = 1 - ai$, which simplifies to $2bi - b = 1 - ai$.

\#\# Step 5: Combine like terms to obtain $2bi - b = 1 - ai$.

\#\# Step 6: Group terms with $i$ on one side to get $-b + 2bi = 1 - ai$, then $-b + 2bi = 1 - ai$.

\#\# Step 7: Compare both sides' constant terms to get $-b = 1$ and $2b = a$.

\#\# Step 8: Solve for $b$ to get $b = -1$ and $a = 2$.

\#\# Step 9: Plug in $a$ and $b$ into the original equation $z = bi$, so $z = -i$.

\#\# Step 10: Calculate the absolute value of $|z + a|$ or $|a + z|$ since they are equivalent: $|-1 + 2i|$.

\#\# Step 11: Take the absolute value of $|2 - i|$, which simplifies to $\sqrt{(-1)}^{2} + 2^{2}$.

\#\# Step 12: Substitute the values to get $\sqrt{1 + 4}$.

\#\# Step 13: Simplify to get $\sqrt{5}$.

\#\# Step 14: The absolute value of $z + a$ is $\sqrt{5}$.
The final answer is:$\boxed{\sqrt{5}}$
     \\
    \textbf{PACE Negative:} \\
    \#\# Step 1: Given $z$ is a pure imaginary number, let's express it in terms of $i$ since it is made up only of imaginary part, so $z = bi$, where $b$ is a real number.

\#\# Step 2: The equation provided is $(2+i)z=1+ai^{3}$. We can substitute $z$ with $bi$ to get $(2+i)bi = 1 + ai^{3}$.

\#\# Step 3: Simplify $i^{3}$ to $-i$ because $i^{3} = i \times i^{2}$ and $i^{2} = -1$, so $i^{3} = -i$. Thus, we get $(2+i)bi = 1 - ai$.

\#\# Step 4: Distribute $(2+i)$ to get $2bi +ibi = 1 - ai$, which simplifies to $(2+i)z = (2+i)(bi) = 1 - ai$. Since $i^{2} = -1$, we have $2bi+bi^{2} = 1 - ai$, then $2bi-b = 1 - ai$.

\#\# Step 5: Combine like terms to obtain $2bi-b+bi = 1 - ai$. Since $i^{2} = -1$, simplify to get $2bi - b + bi = 1 - ai$, or $3bi-b = 1 - ai$.

\#\# Step 6: Group terms with $i$ on one side to get $b(3i + i) = 1 - ai$, which equals $b(3i + i) = 1 - ai$, then $4bi = 1 - ai$.

\#\# Step 7: Compare both sides' coefficients of the imaginary unit, so $bi$ equals $a$ and $4b = a$. Compare constant terms and we have $b = i$, given that $a$ is real and in our $i$ case $b$ must equal $a$ or $b$ must satisfy $4b = a$.

\#\# Step 8: Plug in b = i, so we have $4i = a$, getting $a$ = 4i.

\#\# Step 9: Plug the calculated a and b into the original equation z = bi, so $z = bi = 4i*i = 4i^2 = -4$

\#\# Step 10: In our $z$ equation, calculate the absolute value by $|z - a|$ or $|a + z|$ since they are equivalent: $|4i + (-4)|$ or $|-4 + 4i|$.

\#\# Step 11: Since $|a+ z|$ equals $|a + b*i|$- where bi equals bi changable.

\#\# Step 12: Plug $a$ and $z$'s value to substitute for $a$ and $z$, which becomes $|-4 + 4i|$

\#\# Step 13: Now we take the absolute value of $|-4 + 4i|$, which simplifies to $(4^{2} + 4^{2})^{\frac{1}{2}}$.

\#\# Step 14: Substitute the values to get $(16+16)^{\frac{1}{2}} = ({32})^{\frac{1}{2}}$.

The final answer is: $\boxed{4\sqrt{2}}$
\end{tcolorbox}
\caption{The failure of \method in Llama. We have found that even when the answer is correct, there may still be errors in the process (PACE Reflection Step 10, |$z + a$| should be |$2 - i$|).}\label{fig: case3}
\end{figure*} 




\end{document}